\documentclass[letterpaper,11pt]{article}
\usepackage{amsmath}
\usepackage{amsthm}
\usepackage{amssymb}
\usepackage{graphicx}
\usepackage{tabularx}
\usepackage{xspace}
\usepackage[ruled,vlined]{algorithm2e}
\usepackage{color}

\newtheorem{thmdefinition}{Definition}
\newtheorem{thmlemma}[thmdefinition]{Lemma}
\newtheorem{thmtheorem}[thmdefinition]{Theorem}

\newenvironment{theorem}{\smallskip\begin{thmtheorem}\it}{\end{thmtheorem}\smallskip}
\newenvironment{lemma}{\smallskip\begin{thmlemma}\it}{\end{thmlemma}\smallskip}
\newenvironment{definition}{\smallskip\begin{thmdefinition}\it}{\end{thmdefinition}\smallskip}

\newcommand{\edge}[3]{{#1}\overset{#2}{\longrightarrow}{#3}}

\makeatletter
\newcommand*{\Relbarfill@}{\arrowfill@\Relbar\Relbar\Relbar}
\newcommand*{\xeq}[2][]{\ext@arrow 0055\Relbarfill@{#1}{#2}}
\makeatother

\newcommand*{\probleminternal}[4]{
	\par
	\medskip
	\noindent\qquad\quad\fbox{\parbox{0.9\columnwidth}{
		\textbf{#4: #1} \\[0.05in]
		\renewcommand{\tabcolsep}{2pt}
		\begin{tabularx}{\linewidth}{rX}
			\emph{Input:} & #2 \\
			\emph{Output:} & #3
		\end{tabularx}
	}}
	\par%
	\medskip%
	\par\noindent%
}
\newcommand*{\Mainproblem}[3]{\probleminternal{#1}{#2}{#3}{Problem}}

\newcommand{\Integer}{\mathbb{Z}}
\newcommand{\Natural}{\mathbb{N}}

\newcommand{\length}{\operatorname{length}}
\newcommand{\merge}{\operatorname{merge}}

\newcommand{\cins}{c_{\rm ins}}
\newcommand{\cdel}{c_{\rm del}}
\newcommand{\csub}{c_{\rm sub}}
\newcommand{\la}{\leftarrow}

\newcommand{\dey}{d_e^{(y)}}
\newcommand{\dhy}{d_h^{(y)}}

\newcommand{\changed}[1]{\textcolor{black}{#1}\xspace}
\newcommand{\mistake}[1]{\textcolor{black}{#1}\xspace}
\newcommand{\ambiguty}[1]{\textcolor{black}{#1}\xspace}

\begin{document}
\title{Improper Filter Reduction}
\author{%
  Fatemeh Zahra Saberifar \thanks{Department of Mathematics and
    Computer Science, Amirkabir University of Technology, Tehran,
    Iran. \texttt{fz.saberifar@aut.ac.ir}}%
  \qquad%
  Ali Mohades \thanks{Department of Mathematics and
    Computer Science, Amirkabir University of Technology, Tehran,
    Iran. The corresponding author is A. Mohades. \texttt{mohades@aut.ac.ir}}\\%
  Mohammadreza Razzazi \thanks{Department of Computer Engineering and IT, Amirkabir
    University of Technology, Tehran, Iran. \texttt{razzazi@aut.ac.ir}}%
  \qquad%
  Jason M. O'Kane \thanks{Department of Computer Science and Engineering,
    University of South Carolina, Columbia, South Carolina, USA. \texttt{jokane@cse.sc.edu}}%
  }

\date{ }
\maketitle

\begin{abstract}
  Combinatorial filters have been the subject of increasing interest from the
  robotics community in recent years.
  This paper considers automatic reduction of combinatorial filters to a given
  size, even if that reduction necessitates changes to the filter's behavior.
  We introduce an algorithmic problem called \emph{improper
  filter reduction}, in which the input is a combinatorial filter $F$ along
  with an integer $k$ representing the target size.  The output is another
  combinatorial filter $F'$ with at most $k$ states, such that the difference
  in behavior between $F$ and $F'$ is minimal.

  We present two metrics for measuring the distance between pairs of filters,
  describe dynamic programming algorithms for computing these distances, and
  show that improper filter reduction is NP-hard under these metrics.
  We then describe two heuristic algorithms for improper filter reduction, one
  \changed{greedy sequential} approach, and one randomized global approach based on prior work
  on weighted improper graph coloring.  We have implemented these algorithms
  and analyze the results of three sets of experiments. 
\end{abstract}

\bigskip


\bigskip

\section{Introduction}\label{sec:intro}
This paper builds upon the ongoing line of research on \emph{combinatorial
filtering} for robot tasks~\cite{LaValle12,yu12shadow,Tov09}.
The intuition of that work is to carefully design filters that robots or other
autonomous agents can use to retain only the information that is strictly
necessary for completing an assigned task.
This class of filters is an extremely general tool, applicable for any task
that can be characterized by discrete transitions triggered by finite sets of
observations.

Recent research has considered the problem of automatic reduction of
combinatorial filters: Given a combinatorial filter that correctly solves a
problem of interest, can we algorithmically find the \emph{smallest} equivalent
filter?
Prior work, to which one of the present authors contributed, showed that this
problem is NP-hard and described an efficient heuristic algorithm for finding
small---but not necessarily the smallest---equivalent filters~\cite{OKaShe13}.
However, that prior work left a number of important questions unanswered,
including these:
\begin{quote}
  Is there a useful notion of the behavior of a reduced filter being ``close
  enough'' to the original filter?  If so, then given an input filter, can we
  find the most similar filter of a given fixed size?
\end{quote}
We refer to this problem as the \emph{improper filter reduction} problem, by
analogy to the existing work in improper graph coloring~\cite{Araujo}.
This paper addresses that problem, making several new contributions.
\begin{enumerate}
  \item We introduce a family of metrics for measuring the difference between
  filters and describe dynamic programming algorithms for computing these metrics.

  \item We argue that, for any metric, the improper filter reduction
  problem is NP-hard.

  \item We present two heuristic algorithms for improper filter reduction.
  These are heuristic algorithms, in the sense that there is no guarantee that
  their output will fully minimize the distance from the original filter.
  
  \item We present implementations of these algorithms, along with a series of
  experiments evaluating their performance.
\end{enumerate}
In broad terms, these answers are relevant because they provide some insight
into one of the fundamental tasks in robot design, namely understanding how a
robot should process and retain information collected from its sensors.
\begin{figure}[t]
  \begin{center}
    \raisebox{0.75in}{
      \includegraphics{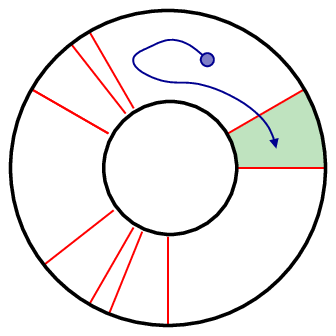}
    }
    \includegraphics[scale=0.45]{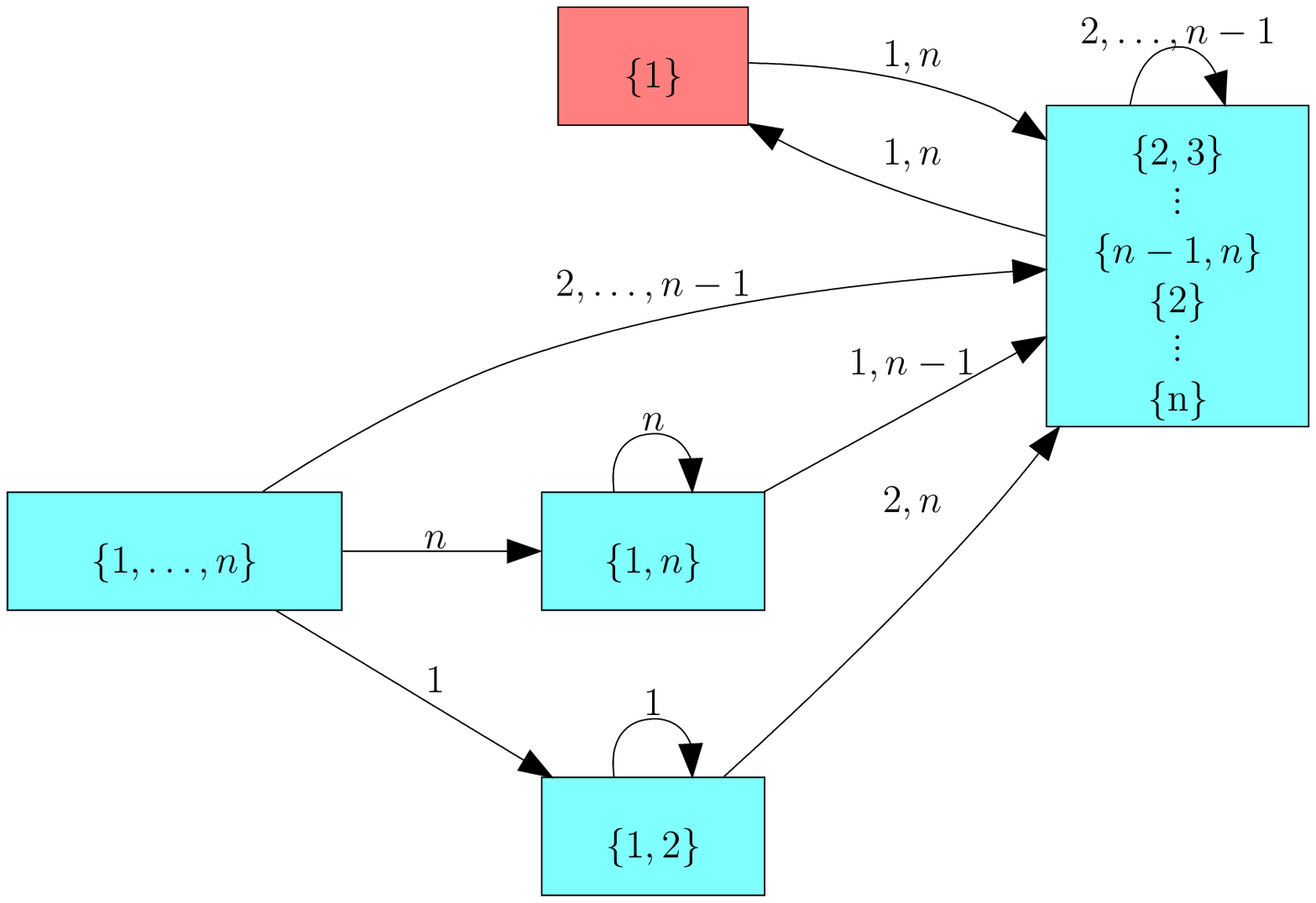}

    \vspace{0.1in}

    \includegraphics[scale=0.45]{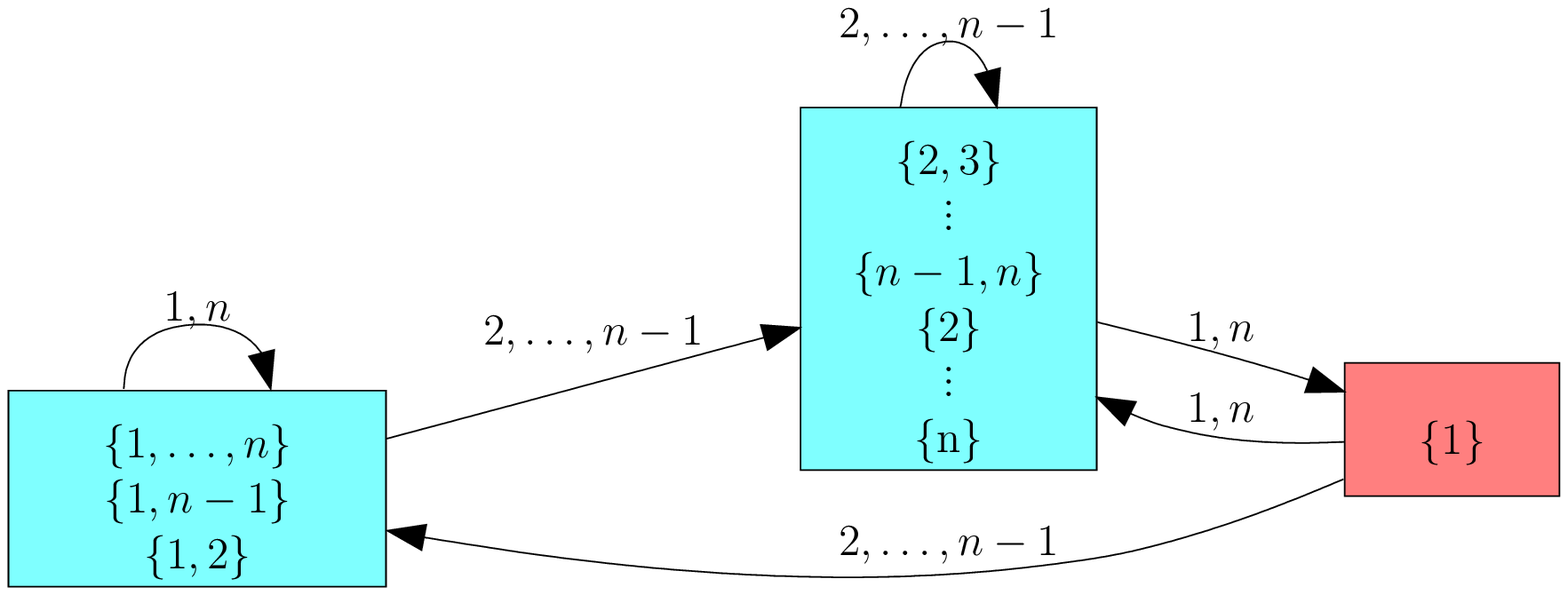}
  \end{center}
  \caption{[top left] An agent moves through an annulus divided into $n$
  regions by $n$ beam sensors.  [top right] The smallest filter, found by a
  prior algorithm~\cite{OKaShe13}, that can accurately track whether the agent
  is provably within region $1$ or not.  [bottom] A smaller filter produced,
  for any specific value of $n$, by the algorithm introduced in this paper.
  This reduced filter eliminates the extra states for region sets $\{1,n\}$ and
  $\{1, 2\}$, which are used only at the start.}
  \label{fig:intro}
\end{figure}
Figure~\ref{fig:intro} shows a simple example, in which an agent moves through
a ring-shaped environment along a continuous but unpredictable path.  A
collection of $n$ beam sensors throughout that environment can detect the
agent passing by, but cannot determine whether that crossing was in the clockwise or
counterclockwise direction.  One pair of beams delimits a special ``goal''
region, called region 1.  A natural question for this system is to ask what
kinds of filters can determine when the agent is within region 1.  That is, we
are interested in designing filters that produce outputs, informally called
colors, that indicate whether the agent is known to be in region 1 or not.
In fact, it is possible to form a series of filters for this problem, each
smaller than the last.
\begin{itemize}
  \item A straightforward approach would be to from a filter,
  represented as a graph of $2^n-1$ nodes, each representing a nonempty set of
  ``possible states,'' with directed edges indicating how the possible states
  change with each beam crossing.  In that filter, we would choose the state
  corresponding to $\{1,\ldots, n\}$ as the initial state.  The filter's output
  would be defined as ``yes'' for the state corresponding to the set
  $\{1\}$---indicating that only state $1$ is consistent with the history of
  observations---and ``no'' for each other state.

  \item That na\"\i{}ve filter can be made more compact by noticing that, if
  $n>3$, only $2n+1$ of those nodes can ever be reached: One initial state in
  which all regions are possible states; one state for each beam, in which only
  the two regions on opposite sides of that beam are possible states;
  and one state for each region, in which the agent is known to be in that
  region.  The remaining $2^n-2n-2$ nodes can be discarded without changing the
  filter's outputs, because no observation sequence can reach them.

  \item That filter can be reduced even further, again without changing its
  outputs, by a carefully selected sequence of vertex contractions.  Prior work
  shows that selecting those vertex contractions optimally is, in general,
  NP-hard~\cite{OKaShe13}, but also presents an efficient algorithm that
  performs this reduction well in practice.  The resulting filter, which has 5
  states regardless of the number of beams, is shown in the top portion of
  Figure~\ref{fig:intro}.

  \item This paper is concerned with additional reductions beyond this smallest
  equivalent filter, which requires us to tolerate some ``mistakes'' in which
  the filter produces incorrect outputs for some observation sequences.  Thus,
  the reduction is ``improper.'' The bottom portion of Figure~\ref{fig:intro}
  shows how the algorithm introduced in this paper can reduce the filter to
  three states, corresponding to ``in region 1''(shown on the right), ``not in
  region 1'' (shown in the middle), and ``maybe'' (shown on the left). The
  filter may produce incorrect outputs for some observation sequences---One
  such observation sequence is $1,n,n,1,1,n,n,\ldots$ --- but behaves the same as
  the original after the first beam not adjacent to region $1$ is crossed.
\end{itemize}
This paper examines the algorithmic problems latent in the final step of this
series of reductions.  The underlying objective is to understand the tradeoffs
between a filter's size and its ability to produce correct outputs.

The remainder of this paper obeys the following structure.
Section~\ref{sec:related} reviews related work, and Section~\ref{sec:prob}
defines the improper filter reduction problem.  Section~\ref{sec:dp} describes
an algorithm for computing the distance between two given filters, which is
used as a subroutine in our two primary algorithms for the improper filter
reduction problem, which are described in Sections~\ref{sec:greedy-reduce} and
\ref{sec:global-reduce}.  Section~\ref{sec:exp} describes our
implementations and experimental evaluation of these algorithms.  Concluding
remarks, including a preview of future work, appear in Section~\ref{sec:conc}.
\section{Related Work}\label{sec:related}
\subsection{Combinatorial filters}
Combinatorial filters, as a general class, build upon the generalized
information space formalism popularized by LaValle~\cite{Lav06,LaValle12}.  The
central idea is to perform filtering tasks in very small derived information
spaces, thereby minimizing the computational burden of executing the filter, and
illuminating the structure underlying the problem itself.
Recent work describes combinatorial filters for
navigation~\cite{lopez12optimal,Tov09}, target tracking~\cite{yu12shadow}, and
story validation~\cite{YuLav11ICRA,YuLav11STAR} problems.
Tovar, Cohen, Bobadilla, Czarnowski, and
LaValle~\cite{tovar09beams,TovarCBCL14} introduced optimal combinatorial
filters for solving some inference tasks in polygonal environments with beam
sensors and obstacles.
Kristek and Shell~\cite{kristek12} showed how to extend existing methods for
sensorless manipulation~\cite{ErdMas88,Gol93} to deformable objects.
Song and O'Kane~\cite{SonOKa12} investigated extensions to limited classes of
infinite information spaces

The deterministic filters we define in this paper that represent as I-state
graph, are an special case of nondeterministic graphs explored by Erdmann~\cite{erdmann10topology,erdmann12topology}.
He has studied topological planning with uncertainty.
\subsection{Filter reduction}
The common thread through all of the prior work mentioned so far is to rely on
human analysis generate efficient filters for specific problems.
The first results on finding optimal filters automatically were presented by
O'Kane and Shell~\cite{OKaShe13}. They proved that the filter minimization problem
is NP-hard and presented a heuristic algorithm to solve it.  The same authors
used similar techniques to solve concise planning~\cite{OKaShe13b} and discreet
communication~\cite{OKaShe15} problems.  Our work extends those filter
reduction results to consider the case in which the reduced filter need not be
strictly equivalent to the original.

Our problem also has some similarity to the problem of measuring the similarity 
of two
deterministic finite automata, as considered by Schwefel, Wegener, and
Weinert~\cite{Schwefel03}, who solved it using evolutionary algorithms.
Chen and Chow also used a scheme for comparing automata, specifically in the
context of web services~\cite{ChenChow2010}.  Our work is differs because of
the unique challenges inherent in the differences between filters and
automata---because the behavior of a filter may be undefined for certain
state-action pairs, the problem of measuring similarities between filters is
more challenging.
\subsection{Probabilistic methods}
There are also some connections between the combinatorial filters considered
here and the Bayesian probabilistic filters commonly used in mobile
robotics~\cite{Thrun05}, including the Kalman filter~\cite{Kalman60} as a
special case.  In both cases, the filter's operation can be described as a
discrete-time transition system, in which the state of the system corresponds
to the robot's ``belief'' about the current state of the world, and transitions
between such states are triggered by observations.  The primary difference is
that for combinatorial filters, we generally assume that both the state space
and the observation space are finite, which enables a number of interesting
algorithmic questions---including, for example, the filter reduction problem
addressed in this paper---to be reasonably posed.

There exists some research in automatic probabilistic motion planning using
partially-observable Markov decision process (POMDP) models.
Roy, Gordon, and Thrun applied dimensionality-reduction techniques to reduce
the computation needed for effective planning under such
models~\cite{roy05belief}.
Ballesteros, Wegener, and Weinert~\cite{Ballesteros13} improved the efficiency
of online POMDPs in planning task by reduction of similar belief points based
on a similarity measurement.  Crook, Keizer, Wang, Tang, and
Lemon~\cite{CrookKWTL14} also presented an automatic POMDP-based approach for
belief space compression. 
\section{Definitions and Formulation}\label{sec:prob}
This section provides basic definitions
for combinatorial filters and the improper filter reduction problem.
\subsection{Filters and their languages}
A robot receives a discrete sequence of observations, drawn from a finite
observation space.  Each observation represents a single sensor reading or a
passively-observed action taken by the robot.  In response to each of these
observations, the robot produces an output, informally called a color, and
modeled without loss of generality as a natural number. 
We model this type of system as a transition graph.  Each vertex of the
transition graph represents the knowledge, called an \emph{information
state}, retained by the robot at some particular time.  Each directed
edge is labeled with an observation, showing how the information state changes
in response to incoming observations.  Definition~\ref{def:isg} formalizes this idea.
\begin{definition}\label{def:isg}
  {A filter $F=(V, E, l\changed{, Y}, c, q_0)$ is 6-tuple, representing a directed
  graph with labels on both its vertices and edges, in which
    \begin{itemize}
      \item the set $V$ is a finite set of vertices called \emph{information
      states} or simply \emph{states},
      \item the multiset $E$ is a finite collection of ordered pairs of states,
      called \emph{transitions}, \mistake{in which every state has at least one out-edge,}
      \item the function $l: E \rightarrow Y$ assigns a label $l(e)$ to each edge
      $e \in E$, such that no two edges share both a source vertex and a label,
      \item the domain $Y$ of $l$ is called the \emph{observation space},
      representing the sensor observations that act as the inputs to the
      filter,
      \item the function $c: V \rightarrow \Natural$ assigns an integer $c(q)$,
      informally called the \emph{color} of $q$, to each
      state $q \in V$, representing the output of the filter at that state,
      and
      \item the vertex $q_0 \in V$ is called the \emph{initial state}.
    \end{itemize}
  }
\end{definition}
\noindent Definition~\ref{def:isg} makes two important assumptions about the
edges in a filter.  First, it requires that no two edges originating
from the same vertex can share the same label.  Thus, given a ``current''
state and a new observation, there is at most one resulting state reached
by following the edge labeled with that observation.  Second, the definition
does \textsl{not} require that there must be an outgoing edge for each
observation from each vertex.  This corresponds to situations in which, based
on the underlying structure of the problem, the robot is certain that a given
observation cannot occur from a given state.

During its execution, the filter starts at the initial state $q_0$ and immediately outputs
$c(q_0)$.  After each observation $y$, the filter transitions from its current
state $q$ to a new state $q'$, following the edge $\edge{q}{y}{q'}$, if that
edge exists.  The filter then generates a new output, namely $c(q')$, and
awaits the next observation.  For a given observation string $s=y_1y_2\cdots y_m$, 
there are two cases:
\begin{enumerate}
  \item If all of the corresponding edges exist, then filter's output is a
  sequence of $m+1$ colors.  We let $\changed{F(s, q)}$ denote the output sequence
  generated when the filter $F$ processes $s$ starting from state $q$.  We also
  use the shorthand $F(s)$ to denote $F(s, q_0)$.

  \item If the filter ever reaches a current state for which no out-edge is labeled
  with the next observation, we say that the filter has failed for this input,
  and the filter output is undefined.
\end{enumerate}
The basic goal of this paper is to understand, given a filter that acts as a
``specification'' of the desired output, how to find small filters that produce
the substantially similar outputs.  However, this comparison only makes sense
for observation sequences for which the output of the original filter is
well-defined.  The next definition formalizes this idea.
\begin{definition}
  {The \emph{language} $L(F)$ of a filter $F$ is the set of all
  observation sequences $s$ for which $F(s)$ is well-defined.}
\end{definition}
\subsection{Defining distance between filters}
Given two filters $F_1$ and $F_2$ with the same observation space $Y$, we
define the distance between $F_1$ and $F_2$ by considering their operation on
identical observation strings, and quantifying the difference between the
corresponding output strings.

Specifically, we choose a function
  $ m: \changed{\bigcup}_{i \in \Natural} (Y^i \times Y^i) \to \Integer$,
representing a metric that measures the distance between two equal-length
observation strings.  We consider, both in the algorithms of
Section~\ref{sec:dp} and the experiments in Section~\ref{sec:exp}, two specific
options for $m$:
\begin{enumerate}
  \item The Hamming distance~\cite{Hamming}, denoted $\operatorname{h}$, which
  counts the number of positions at which the two strings differ.
  \item The edit distance~\cite{Wagner–Fischer}, denoted $\operatorname{e}$,
  which is the smallest number of single-character insert, delete, and
  substitute operations, weighted by given operation costs $\cins$, $\cdel$,
  and $\csub$ respectively, needed to transform \changed{one} string into another.
\end{enumerate}
We then use this distance function over observation strings, either $m=h$
or $m=e$, to define the distance $D_m(F_1, F_2)$ between a pair of filters
$F_1$ and $F_2$:
  \begin{equation}\label{eq:D}
    D_m(F_1, F_2)  = \sup_{s \in L(F_1) \cap L(F_2)} \left(
      \frac
        {m(F_1(s), F_2(s))}
        {\length(s)+1}
    \right).
  \end{equation}
The intuition is to consider the worst case distance between output strings, over all
observation strings in $L(F_1)\cap L(F_2)$---that is, over all observation
strings that can be processed by both filters---normalized by the length of
the output strings.
For the special case in which $L(F_1)\cap L(F_2) = \emptyset$, we define
$D_m(F_1,F_2)=0$.
The
use of worst-case reasoning allows us to compare filters without the modelling
burden of assigning probabilities to observation sequences, and the
normalization is necessary to ensure that the distance between filters is
finite.  Because the denominator represents the length of both $F_1(s)$ and
$F_2(s)$---that is, one more output than there are observations---this can be
viewed as an ``average cost-per-stage'' model as described by
LaValle~\cite{Lav06}.

\subsection{Improper filter reduction}
We can now state the central algorithmic problem addressed in this paper.

\Mainproblem{Improper-FM}
  {A filter $F$ and an integer $k$.}
  {A filter $F'$ with at most $k$ states, such that $L(F) \subseteq L(F')$ and
  $D_m(F, F')$ is minimal.}%

Note, however, that this problem can be proven NP-hard, regardless of the
choice of $m$.

\begin{theorem}
  \textsc{Improper-FM} is NP-hard.
\end{theorem}
\begin{proof}
  Reduction from the basic (error-free) filter minimization problem,
  \textsc{FM}~\cite{OKaShe13}.  Given an instance $\changed{F}$ of \textsc{FM}, one can
  use binary search on the range $\{0, \ldots, n\}$ to find the smallest value
  of $k$ for which there exists an $F'$ with $L(F) \subseteq L(F')$ and $D_m(F, F')=0$.  This $F'$ is, by
  definition, the correct output for \textsc{FM}.  Therefore, a polynomial time
  algorithm for \textsc{Improper-FM} would imply a polynomial time algorithm for
  \textsc{FM}.  Since \textsc{FM} is NP-hard, and we have a polynomial-time
  reduction from \textsc{FM} to \textsc{Improper-FM}, conclude that
  \textsc{Improper-FM} NP-hard as well.
\end{proof}

Consequently, we restrict our attention \changed{in} this paper to heuristic algorithms
that attempt, but cannot guarantee, to minimize the distance between the reduced
filter and the original.

\section{Computing distance between filters}\label{sec:dp}
Before turning our attention to \textsc{Improper-FM}, we first consider the
related problem of computing the distance $D_m(F_1, F_2)$ between a given pair of
filters $F_1$ and $F_2$.  Attempting to evaluate Equation~\ref{eq:D} directly
would be futile, because it includes a supremum over $L(F_1)\cap L(F_2)$, which
in general may be an infinite set.  Instead, we introduce a dynamic programming
algorithm that whose outputs converge to $D_m(F_1, F_2)$.  The details of the
algorithm differ depending on whether the underlying string distance function
$m$ is Hamming distance function $h$ (addressed in Section~\ref{sec:dp-ham}) or
edit distance function $e$ (addressed in Section~\ref{sec:dp-edit}).

\subsection{Hamming distance based metric}\label{sec:dp-ham}
Our algorithm for computing $D_h(F_1, F_2)$ is based on computing
successive values of a simpler function $d_h$, which only considers observation
strings up to a certain given length.  Our algorithm considers longer
observation strings at each iteration.  To facilitate a dynamic programming
solution, we also must consider different starting states for each filter.

\begin{definition}
  Let $d_h(q_1, q_2, k)$ denote the maximum, over all strings $s$ of length $k$
  in $L(F_1) \cap L(F_2)$, of $h(F_1(\changed{s, q_1}), F_2(\changed{s, q_2}))$, or $0$ if there
  are no such strings.
\end{definition}

\noindent The function $d_h$ is useful because values for $D_h$ can be derived
from values of $d_h$, as shown in the next lemma.

\begin{lemma}\label{lem:Dlim}
  For any two filters $F_1$ and $F_2$, with initial states $q_{0_1}$ and
  $q_{0_2}$ respectively, and with $L(F_1) \subseteq L(F_2)$, we have
    \begin{equation*}\label{eq:Dlim}
      D_h(F_1, F_2) = \lim_{k \to \infty} \left(
        \max_{i \in 1,\ldots,k} \frac{d_h(q_{0_1}, q_{0_2}, i)}{i+1}
      \right).
    \end{equation*}
\end{lemma}
\begin{proof}
  For any $\epsilon>0$, we must show that, for sufficiently large $k$, we have
    $$ \left|
        D_h(F_1, F_2) -
        \max_{i \in 1,\ldots,k} \frac{d_h(q_{0_1}, q_{0_2}, i)}{i+1}
      \right| < \epsilon.
    $$
  First, we note that, for any $k$, we have
  \begin{eqnarray*}
    D_h(F_1, F_2)
    &\ge& \max_{i \in 1,\ldots,k} \frac{d_h(q_{0_1}, q_{0_2}, i)}{i+1},
  \end{eqnarray*}
  because both the left and right expressions are defined to maximize the same
  quantity defined over some set of observation strings, but the left side
  considers a superset of the observation strings considered in the maximum
  operation in the right side.

  For the other direction, we let $s$ denote an observation string for which
    \begin{equation}\label{eq:s}
      D_h(F_1, F_2) - \frac{h(F_1(s), F_2(s))}{\length(s)+1} \le \epsilon.
    \end{equation}
  Such a string must exist by the definition of supremum.
  Then, choosing $k=\length(s)$, we have
  \begin{equation}\label{eq:xxx}
    \max_{i \in 1,\ldots,k} \frac{d_h(q_{0_1}, q_{0_2}, i)}{i+1}
    \ge \frac{h(F_1(s), F_2(s))}{k+1}
    \ge D_h(F_1, F_2) - \epsilon.
  \end{equation}
  In Equation~\ref{eq:xxx}, the first step holds because the specific string
  $s$ is among the strings considered in the maximum over all strings of length
  at most $k$, and the second steps proceeds directly from Equation~\ref{eq:s}.
  Therefore, the difference between $D_h(F_1,F_2)$ and 
  $\max_{i \in 1,\ldots,k} d_h(q_{0_1}, q_{0_2}, i)/({i+1})$
  is at most $\epsilon$, completing the proof. 
\end{proof}

When $k=0$, the filter receives no observations and produces a single output,
so $d_h$ is trivial to compute in this case, depending only on whether the
single outputs produced by each filter match each other:
\begin{equation}\label{eq:HD-base}
  d_h(q_1, q_2, 0) = \begin{cases}
    0 & \text{if $c_1(q_1) = c_2(q_2)$} \\
    1 & \text{otherwise}
  \end{cases}.
\end{equation}

Next we address the general case in which $k > 0$.  From any state pair $(q_1$, $q_2)$, we
must consider the set of observations for which both $q_1$ and $q_2$ have
\changed{same} outgoing edges, denoted $Y(q_1, q_2)$.  Thus, for any observation in $Y(q_1,
q_2)$, we know that there exists an edge $\edge{q_1}{y}{q_1'}$ in $F_1$ and an edge
$\edge{q_2}{y}{q_2'}$ in $F_2$, labeled with the same observation $y$.  
 \begin{equation}
    \changed{ \dhy(q_1, q_2, k)
    = 
     d_h(q_1, q_2, 0) + d_h(q'_1, q'_2, k-1).}
  \end{equation}
\changed{Then we can express $d_h(q_1, q_2,k)$ recursively in terms of $\dhy$ values with
shorter observation string lengths:}
   \begin{equation}\label{eq:general-h}
    \changed{d_h(q_1, q_2, k)
    = \max_{y \in Y(q_1, q_2)} \dhy(q_1, q_2, k),}
   \end{equation}
When $Y(q_1, q_2)$ is empty, we use $d_h(q_1, q_2, k) = 0$.  (The intuition of
this special case is that, in this case, $q_1$ and $q_2$ are a good choice to
merge, because when forming a reduced filter in which $q_1$ is merged with
$q_2$, there will be no ambiguity in selecting the correct destination for any
transition from the combined state.)

Algorithm~\ref{alg:Dh} shows the dynamic programming algorithm that uses
this recurrence to compute $D_h(F_1, F_2)$.
The intuition is to compute $d_h$ for increasing values of $k$, until those
$d_h$ values converge, and then to find appropriate normalized maximum.

\begin{algorithm}[t]
  \label{alg:Dh}
  \caption{Dynamic programming to compute $D_h(F_1, F_2)$.}
  \DontPrintSemicolon
  \SetAlgoLined
  \newcommand{\done}{\ensuremath{\mbox{done}}}
  \newcommand{\true}{\ensuremath{\mbox{True}}}
  \newcommand{\false}{\ensuremath{\mbox{False}}}
  $k \la 0$ \;
  \Repeat{\done}{
    $\done \la \true$\;
    \For{$(q_1, q_2) \in V_1 \times V_2$}{
      Compute $d_h(q_1, q_2, k)$ using Eq.~\ref{eq:HD-base} or \ref{eq:general-h}.\;
      \If{$\left|
        \frac{d_h(q_1, q_2, k)}{k+1}
        -
        \frac{d_h(q_1, q_2, k-1)}{k}
      \right| > \epsilon$}{
        $\done \la \false$ \;
      }
    }
    $k \la k + 1$ \;
  }
  \Return{$\displaystyle \max_{i=0,\ldots,k-1} \frac{d_h(q_{0_1}, q_{0_2}, i)}{i+1}$}
\end{algorithm}

\subsection{Edit distance based metric}\label{sec:dp-edit}
We now turn our attention to algorithms that, given two filters $F_1$ and
$F_2$, compute $D_e$.  The approach is similar to the approach for $D_h$ in
Section~\ref{sec:dp-ham}.  However, the algorithm for $D_e$ is more complex
because it must account for the different-length strings that can arise from
insert and delete operations on the output sequences.

The algorithm works by computing values for a function called $d_e$, which we
define below.
\begin{definition}\label{def:de}
  Let $d_e(q_1, k_1, q_2, k_2)$ denote the largest edit distance\changed{, denoted $e(F_1(s_1, q_1), F_2(s_2, q_2))$,}
 between
    $F_1(\changed{s_1, q_1})$ and $F_2(\changed{s_2, q_2})$,
  over all
  strings $s_1$ of length $k_1$ and all
  strings $s_2$ of length $k_2$,
  for which 
    $s_1,s_2 \in L(F_1)\cap L(F_2)$, 
  and $s_1$ and $s_2$ are identical to each other for the first $\min(k_1,k_2)$
  observations.
\end{definition}
We can show that $d_e$ is useful for computing $D_e$ with a lemma analogous to
Lemma~\ref{lem:Dlim}.
\begin{lemma}\label{lem:Dlim2}
  For any two filters $F_1$ and $F_2$, with initial states $q_{0_1}$ and
  $q_{0_2}$ respectively, and with $L(F_1) \subseteq L(F_2)$, we have
    \begin{equation*}\label{eq:Dlim2}
      D_e(F_1, F_2) = \lim_{k \to \infty} \left(
        \max_{i \in 1,\ldots,k} \frac{d_e(q_{0_1}, i, q_{0_2}, i)}{i+1}
      \right).
    \end{equation*}
\end{lemma}
\begin{proof}
  The same argument from the proof of Lemma~\ref{lem:Dlim} applies.  The fact
  that $d_e$ considers observation sequences of different lengths is irrelevant
  here, because the limit in Equation~\ref{eq:Dlim2} uses the same value,
  namely $k_1=k_2=i$, for the two string lengths. 
\end{proof}

We can now construct a recurrence for $d_e$.  There are three base cases.
\begin{enumerate}
  \item When $k_1=0$ and $k_2=0$, both filters produce a single
  output.  Editing one such string into the other can be one either by a
  substitution, or by one deletion and one insertion:
    \begin{equation}\label{eq:ED-baseA}
      d_e(q_1, 0, q_2, 0) = \begin{cases}
        0 & \text{if $c_1(q_1) = c_2(q_2)$} \\
        \min(\csub, \cdel + \cins) & \text{otherwise}
      \end{cases}.
    \end{equation}

  \item When $k_1=0$ and $k_2>0$, we have a single character $c(q_1)$ of output
  from $F_1$ and a longer output string $F_2(s_2,q_2)$, with length $k_2+1$,
  from $F_2$.
  The edit distance between these strings depends on whether $c(q_1)$ appears
  in $F(s_2, q_2)$.
  If $c(q_1)$ does not appear in $F(s_2,q_2)$, then we need either
    (a) 1 deletion and $k_2+1$ insertions or
    (b) 1 substitution and $k_2$ insertions,
  whichever has smaller total cost.
  If $c(q_1)$ does appear in $F(s_2, q_2)$, then $c(q_1)$ can be `reused' in
  $F(s_2, q_2)$, reducing the edits to simply $k_2$ insertions.

  Because Definition~\ref{def:de} calls for the largest edit distance, the
  relevant question is whether there exists \textsl{any} sequence of observations
  which, when given as input to $F_2$ starting at $q_2$, avoids all states with
  color $c(q_1)$ for at least $k_2$ transitions.  Let $L(q_2,c(q_1))$ denote
  the length of the longest path in $F_2$ starting from $q_2$ that does not
  visit any states of color $c(q_1)$.\footnote{Note that this need not be a
  simple path, so the standard hardness result for longest paths in direct
  graphs~\cite{GareyJohnson1979} does not apply; in fact we can compute $L(q_2,
  c(q_1))$ efficiently using a variant of Dijkstra's algorithm.}
  We can express $d_e$ for this case based on this value:
    \begin{equation}\label{eq:ED-baseB}
      d_e(q_1, 0, q_2, k_2 ) = k_2\cins + \begin{cases}
        \min(
          \cdel + \cins,
          \csub
        )
          & \text{if $L(q_2,c(q_1)) > k_2$} \\
        0
          & \text{otherwise}
      \end{cases}
    \end{equation}

  \item When $k_1>0$ and $k_2=0$, the situation is analogous, but with the
  roles of $F_1$ and $F_2$ swapped:
    \begin{equation}\label{eq:ED-baseC}
      d_e(q_1, k_1, q_2, 0) = k_1\cdel + \begin{cases}
        \min(
          \cins + \cdel,
          \csub
        )
          & \text{if $L(q_1,c(q_2)) > k_1$} \\
        0
          & \text{otherwise}
      \end{cases}
    \end{equation}

\end{enumerate}
In the general case, we can express values for $d_e$ using a recurrence similar
to the standard recurrence for edit distance~\cite{Wagner–Fischer}, but
accounting for the changes in state that accompany each observation.  For given
values of $q_1$, $k_1$, $q_2$, and $k_2$, we must consider the worst case over
all observations for which both $q_1$ and $q_2$ have \changed{same} outgoing edges.  As in
Section~\ref{sec:dp-ham}, we write $Y(q_1, q_2)$ to denote this set of
observations, and for each $y \in Y(q_1, q_2)$ we write $\edge{q_1}{y}{q_1'}$
and $\edge{q_2}{y}{q_2'}$ for the two corresponding edges.  To compute
$d_e(q_1, k_1, q_2, k_2)$, we must consider the worst case over all
observations $y \in Y(q_1, q_2)$:
  \begin{equation}\label{eq:general-e}
    d_e(q_1, k_1, q_2, k_2)
      = \max_{y \in Y(q_1, q_2)} \dey(q_1, k_1, q_2, k_2),
  \end{equation}
in which we have introduced a shorthand notation $\dey$ for the value of $d_e$
predicated receiving a specific observation $y$ next.
There are two cases for $\dey$.
\begin{enumerate}
  \item If $c_1(q_1) = c_2(q_2)$, no edits are needed, so we have
    $$ \dey(q_1, k_1, q_2, k_2) = \dey(q_1', k_1-1, q_2', k_2-1).$$

  \item If $c_1(q_1) \neq c_2(q_2)$, we use the smallest cost from among
  insert, delete, and substitute operations:
    \begin{equation*}
      \changed{\dey(q_1, k_1, q_2, k_2)} =
      \min \left\{
        \begin{gathered}\label{eq:d2}
          d_e(q_1,   k_1,    q'_2,  k_2-1 ) + \cins, \\
          d_e(q_1',  k_1-1,  q_2,   k_2   ) + \cdel, \\
          d_e(q_1',  k_1-1,  q'_2,  k_2-1 ) + \csub
        \end{gathered}
      \right\}.
    \end{equation*}
  The intuition is that, in the case of \changed{a delete}, $F_1$ will process one
  observation, transitioning from $q_1$ to $q_1'$ and reducing the length of
  the remaining observation string by 1, whereas $F_2$ does not process any
  observations, remaining in state $q_2$, with the same number of observations
  remaining. \changed{For an insert, $F_2$ will process one
  observation, transitioning from $q_2$ to $q_2'$ and reducing the length of
  the remaining observation string by 1, whereas $F_1$ does not process any
  observations, remaining in state $q_1$, with the same number of observations
  remaining.}  
  For a substitution, both $F_1$ and $F_2$ transition to new
  states, and both $k_1$ and $k_2$ are decreased.  This matches the usual
  recurrence for edit distance, with the extra complication that we must
  consider the states reached by the two filters in addition to the string
  lengths.
\end{enumerate}

This recurrence leads directly to an algorithm for computing $D_e(F_1, F_2)$.
Pseudocode appears as Algorithm~\ref{alg:De}.  The algorithm applies the
recurrence, starting from $k_1=k_2=0$, and increasing those indices until the
average error per observation converges, as in Lemma~\ref{lem:Dlim2}. The most
important subtlety is in the ordering:  Before computing $d_e$ values with
input string lengths $(k_1, k_2)$, we must ensure that the corresponding
computations have been completed for $(k_1-1, k_2)$, $(k_1,k_2-1)$, and
$(k_1-1,k_2-1)$.

\begin{algorithm}[t]
  \label{alg:De}
  \caption{Dynamic programming to compute $D_e(F_1, F_2)$.}
  \DontPrintSemicolon
  \SetAlgoLined
  \newcommand{\done}{\ensuremath{\mbox{done}}}
  \newcommand{\true}{\ensuremath{\mbox{True}}}
  \newcommand{\false}{\ensuremath{\mbox{False}}}
  $i \la 0$ \;
  \Repeat{\done}{
    $\done \la \true$\;
    \For{$j \la 0, \ldots, i$}{
      \For{$(q_1, q_2) \in V_1 \times V_2$}{
        Compute $d_e(q_1, i, q_2, j)$ using Eq.~\ref{eq:ED-baseA}, \ref{eq:ED-baseB}, \ref{eq:ED-baseC}, or \ref{eq:general-e}.\;
        Compute $d_e(q_1, j, q_2, i)$ using Eq.~\ref{eq:ED-baseA}, \ref{eq:ED-baseB}, \ref{eq:ED-baseC}, or \ref{eq:general-e}.\;
      }
    }
    \For{$(q_1, q_2) \in V_1 \times V_2$} {
      \If{$\left|
        \frac{d_e(q_1, i, q_2, i)}{i+1}
        -
        \frac{d_e(q_1, i-1, q_2, i-1)}{i}
      \right| > \epsilon$}{
        $\done \la \false$ \;
      }
    }
    $k \la k + 1$ \;
  }
  \Return{$\displaystyle \max_{i=0,\ldots,\changed{k-1}} \frac{d_e(q_{0_1}, i, q_{0_2}, i)}{i+1}$}

\end{algorithm}

Note that Algorithm~\ref{alg:De} is significantly slower than
Algorithm~\ref{alg:Dh} because of the need for an additional nested loop to
accommodate the differing string lengths between $F_1$ and $F_2$.
This difference, which we also observe in the experiments described in
Section~\ref{sec:exp}, confirms the basic intuition that Hamming distance is a
computationally simpler metric for string distance than edit distance.  Note,
however, that edit distance is a more `forgiving' metric, in the sense that if
$\cins=\cdel=\csub=1$, then $D_e(F_1, F_2) \le D_h(F_1, F_2)$ for any two
filters $F_1$ and $F_2$.

\section{Local greedy sequential reduction}\label{sec:greedy-reduce}
In this section, we present a heuristic algorithm for improper filter reduction
that efficiently reduces the input filter to the required size, while
attempting to minimize the error.  The algorithm performs a series of greedy
``merge'' operations, each of which reduces the filter size by one.  We first
describe the details of this merge operation (Section~\ref{sec:merge}), and
then propose an approach for selecting pairs of states to merge
(Section~\ref{sec:reduce}).

\subsection{Merging states}\label{sec:merge}
Suppose we have a filter $F$ with $n$ states, and we want form a new, nearly
identical, filter $F'$ with $n-1$ states.  One way to accomplish this is to
select two states $q_1$ and $q_2$ from $F$---deferring to
Section~\ref{sec:reduce} the question of how to select $q_1$ and $q_2$---and
merge them, in the following way. 
\newcommand{\qot}{q_{1,2}}  
\begin{itemize}

  \item Replace $q_1$ and $q_2$ with a combined state, denoted $\qot$.  If
  either of $q_1$ or $q_2$ is the filter's initial state, then $\qot$ becomes
  the new initial state. \changed{We assign the color of $q_1$ to the combined state $\qot$.}

  \item Replace each in-edge $\edge{q_i}{y}{q_1}$ of $q_1$, with a new edge
  $\edge{q_i}{y}{\qot}$ to the combined state.  Repeat for $q_2$, replacing
  each $\edge{q_i}{y}{q_2}$ with $\edge{q_i}{y}{\qot}$.

  \item Replace each out-edge $\edge{q_1}{y}{q_j}$ of $q_1$ with a new edge
  $\edge{\qot}{y}{q_j}$.

  \item For each out-edge $\edge{q_2}{y}{q_j}$, determine whether $q_1$ has
  an out-edge with the same label $y$.  If so, then discard the edge
  $\edge{q_2}{y}{q_j}$.  If not, then replace that edge with
  $\edge{\qot}{y}{q_j}$.
\end{itemize}
The intuition is to replace the two states $q_1$ and $q_2$ with just one state,
with as little disruption to the filter as possible.  The only
complication---and the reason that this merge operation must be more complex
than a standard vertex contraction---is that the combined node can have at most
one out-edge for each observation label.  When $q_1$ and $q_2$ both have
out-edges with the same label, we arbitrarily resolve that conflict by giving
priority to $q_1$ over $q_2$.

In the context of the \textsc{Improper-FM} problem, we also must ensure that the
filter $F'$ resulting from a merge in $F$ has $L(F) \subseteq L(F')$.  This may
not be immediately true, as illustrated in Figure~\ref{fig:merge-v}.  To
resolve this problem, we perform a forward search over state pairs $(q_a,
q_b)$, in which $q_a$ is from the original filter $F$ and $q_b$ is from the
merged filter $F'$, starting from $(\changed{q_2}, \qot)$ \changed{when we suppose
that $q_2$ is merged into the $q_1$ to create $\qot$}.  Each time this search finds a
reachable state pair $(q_a, q_b)$ and an observation $y$, for which $F$ has an
edge $\edge{q_a}{y}{q_a'}$ but $F'$ has no edge from $q_b$ labeled $y$, we
insert an edge $\edge{q_b}{y}{q_a'}$ into $F'$.  This ensures that any
observation sequence that can be processed by $F$ can also be processed by
$F'$.  The runtime of this post-processing step is quadratic in the number of
states in $F$.

\begin{figure}
  \quad
  \includegraphics{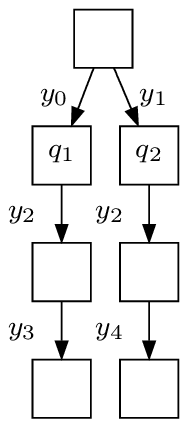}
  \hfill
  \includegraphics{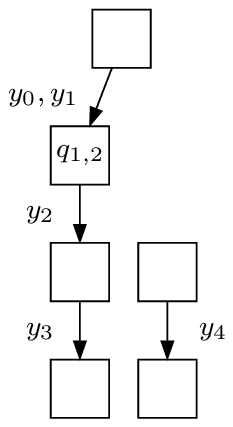}
  \hfill
  \includegraphics{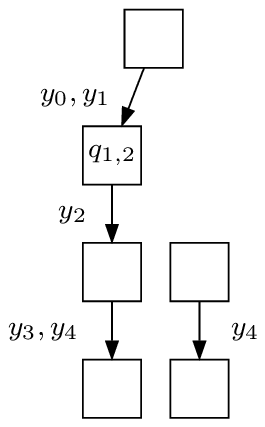}
  \quad
  \null

  \label{fig:merge-v}
  \caption{An illustration of the need for post-processing after merging two
  states $q_1$ and $q_2$, to ensure that $L(F) \subseteq L(F')$.
  [left] A
  filter $F$, before $q_1$ and $q_2$ are merged.
  [center] Another filter $F'$, formed by combining $q_1$ with $q_2$, but
  without post-processing.  In this example, $F$ can process the observation
  string $y_1y_2y_4$, but $F'$ fails on that input.
  [right] The post-processing step adds an edge to resolve the problem.
  }
\end{figure}
We write $\merge(F, q_1, q_2)$ to denote the filter resulting from applying
this complete merge operation to $q_1$ and $q_2$ in $F$.

\subsection{Selecting pairs of states to merge}\label{sec:reduce}
We can use this merge operation to form a heuristic algorithm for improper
filter reduction in a local greedy way.  The basic idea is to consider all
ordered pairs of distinct states in the current filter as candidates to merge,
and to compare the filter resulting from each such merge to the original input
filter.  The merge that results in the smallest distance from the original
filter is kept, and the algorithm repeats this process until filter is reduced
to the desired size.

Beyond this basic idea, we add two additional constraints to improve the
quality of the final solution.  First, we reject any merge operation that
leaves some states unreachable, unless all available merges leave at least one
unreachable state.  Second, we also reject any merge operation that eliminates
at least one output color from the resulting filter, unless all available
merges eliminate a color.  
Algorithm~\ref{alg:Greedy} shows the complete approach.  See
Section~\ref{sec:exp} for an evaluation of this algorithm's effectiveness.

 \begin{algorithm}
   \label{alg:Greedy}
   \caption{A greedy sequential method to reduce $F$ to $k$ states.}
   \DontPrintSemicolon
   \SetAlgoLined
   \newcommand{\failed}{\ensuremath{\mbox{failed}}}
   \newcommand{\true}{\ensuremath{\mbox{True}}}
   \newcommand{\false}{\ensuremath{\mbox{False}}}
   $f \la \false$\;
   $F_{\rm orig} \la F$ \;
   \While{$|V(F)| > k$}{
      $D^\star \leftarrow \infty$ \;
      \For{$(q_1, q_2) \in V(F) \times V(F)$}{
          \If{$q_1 = q_2$}{
            \textbf{continue} \;
          }
          $F' \leftarrow \operatorname{merge}(F, q_1, q_2)$ \;

          \BlankLine
          
          \If{$f = \false$}{
            \If{$F'$ has unreachable states}{
              \textbf{continue} \;
            }
            \If{$F'$ has unreachable colors}{
             \textbf{continue} \; 
            }
          }

          \BlankLine

          $D \la D_m(F_{\rm orig}, F')$ \tcp{Use Alg.~\ref{alg:Dh} or Alg.~\ref{alg:De}.} \;
          \If {$D < D^\star$}{
            {$D^\star \leftarrow D$}\;
            {$F^\star \leftarrow F'$}\;
          }
       }

       \If{$D^\star = \infty$}{
          $f \leftarrow \true$ \;
       }
       \Else{
          $f \leftarrow \false$ \;
          $F \leftarrow  F^\star$\;
       }
   }
   \Return $F$ \;
  
 \end{algorithm}

\section{\changed{Randomized global} reduction}\label{sec:global-reduce}

The central limitation of Algorithm~\ref{alg:Greedy} is that it selects merges
to perform in a sequential way, and therefore cannot account for the impact
that each merge has on later steps in the algorithm.  In this section, we
present an alternative to Algorithm~\ref{alg:Greedy} that avoids this problem
by selecting all of the merges to perform at once, in a global way.
Algorithm~\ref{alg:Global} outlines the approach.
The intuition is to cast the problem of choosing which states to merge with
each other as an improper graph coloring problem, which we solve using an existing
algorithm (Section~\ref{sec:coloring}).  We then apply
a randomized `voting' process to construct the reduced filter from the colored
graph (Section~\ref{sec:voting}).

\begin{algorithm}
  \label{alg:Global}
  \caption{A randomized global method to reduce $F$ to $k$ states using $r$
  iterations.}
  \DontPrintSemicolon
  \SetAlgoLined

  $C \la C(F)$ \;
  $c \la$ improper coloring of $C$ with $k$ colors\cite{Araujo}.

  \BlankLine

  $D^\star \leftarrow \infty$ \;
  \For{$\changed{R} \la 0, \ldots, r$}{
    $F' \la$ empty filter \;
    \For{each color $i$ in $c$}{
      \ambiguty{$q \la $ state in $F$ randomly selected from those colored $i$ in $C(G)$}\;
      Create a state $q_i$ in $F'$ with same output color as $q$. \;
    }
    \For{each color $i$ in $c$}{
      \For{each observation $y$ \changed{$\in Y$} in $F$}{
        \ambiguty{$q \la $ state in $F$ randomly selected from those both colored~$i$ in $C(G)$ and with an edge $\edge{q}{y}{w}$ in $F$.}\;
        Create an edge in $F'$ from $q_i$ to $q_{c(w)}$ labeled $y$.
      }
    }

    \BlankLine

    $D \la D_m(F, F')$ \tcp{Use Alg.~\ref{alg:Dh} or Alg.~\ref{alg:De}.}
    \If {$D < D^\star$}{
      {$D^\star \leftarrow D$}\;
      {$F^\star \leftarrow F'$}\;
    }
  }
  \Return{$F^\star$} 

\end{algorithm}

\subsection{Selecting merges via improper graph coloring}\label{sec:coloring}

We can view the problem of reducing a given filter $F$ down to $k$ states as a
question of partitioning the states of $F$ into $k$ groups.  Our algorithm
accomplishes this by solving a coloring problem on a complete undirected graph
$C(F)$, defined as follows:
\begin{enumerate}
  \item For each state $q$ in $F$, we create one vertex $v(q)$ in $C(F)$.

  \item For each pair of distinct vertices $v(q_1), v(q_2)$ in $C(F)$, we create a
  weighted edge.  The intuition is that the weights should be estimates of how
  different the two corresponding states are.
  If the behavior of $F$ is very similar when started from $q_1$ as when
  started from $q_2$, then we should assign a relatively small weight to the
  edge between $v(q_1)$ and $v(q_2)$.
  Conversely, if the behavior of $F$ differs significantly when started from
  $q_1$ compared to starting from $q_2$, then we should assign a relatively
  large weight to that edge.
  
  We can capture these kinds of differences by computing the distance of $F$
  with itself---That is, by computing either $D_h(F, F)$ or $D_e(F, F)$---and
  examining the intermediate values---either $d_h(q_1, q_2, k)$ or $d_e(q_1, k,
  q_2, k)$, in which $k$ is the number of iterations of the outermost loops in
  Algorithm~\ref{alg:Dh} or Algorithm~\ref{alg:De}---computed along the way.
  Specifically, we assign $w(q_1, q_2)$ as
    $$ w(q_1, q_2) =  
      \lim_{k \to \infty} \left(
        \max_{i \in 1,\ldots,k} \frac{d_h(q_1, q_2, i)}{i+1}
      \right)
    $$
  for Hamming distance, and 
    $$ w(q_1, q_2) =  
      \lim_{k \to \infty} \left(
        \max_{i \in 1,\ldots,k} \frac{d_e(q_1, i, q_2, i)}{i+1}
      \right)
    $$
  for edit distance.
\end{enumerate}
Figure~\ref{fig:donut-plot5-1} shows an example of this construction.
\begin{figure}[t]
  \begin{center}
    \includegraphics[scale=0.5]{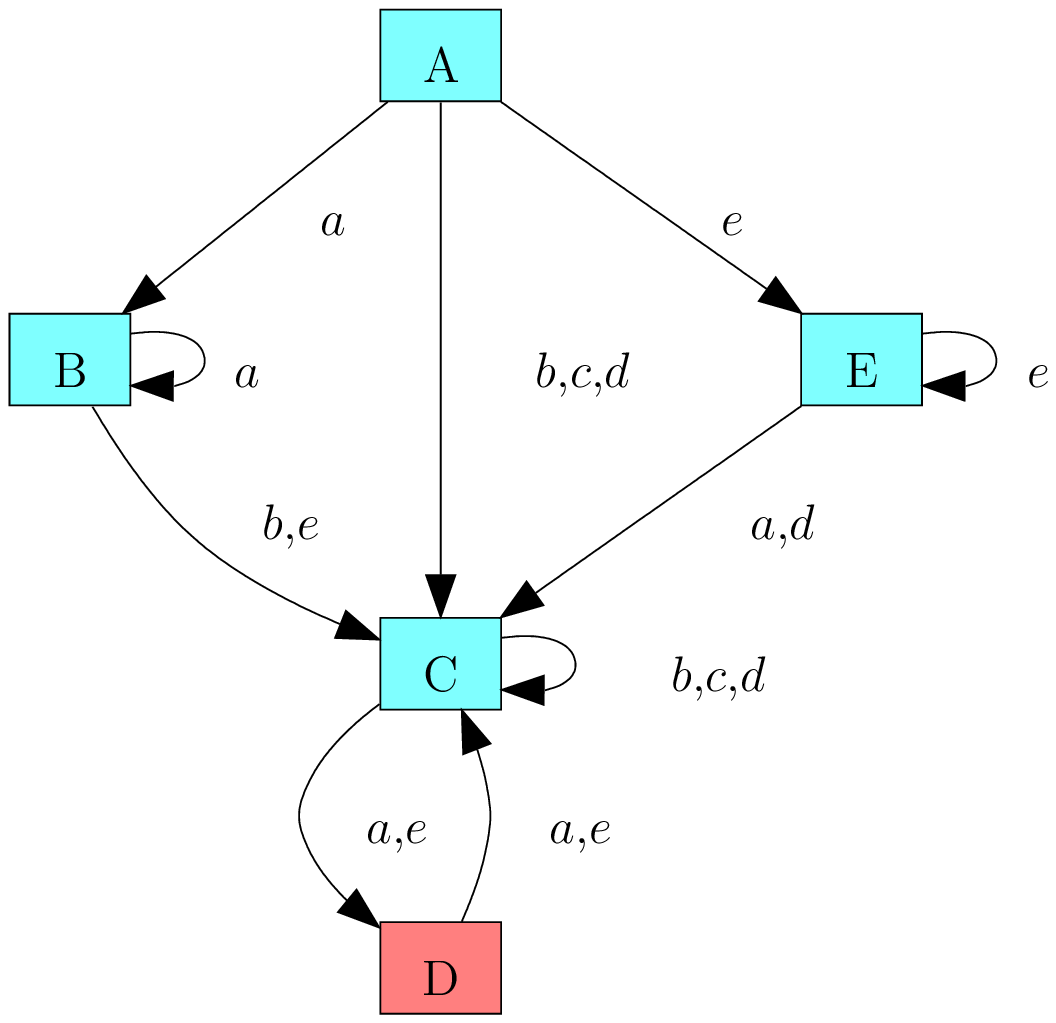}
    \qquad
    \includegraphics[scale=0.5]{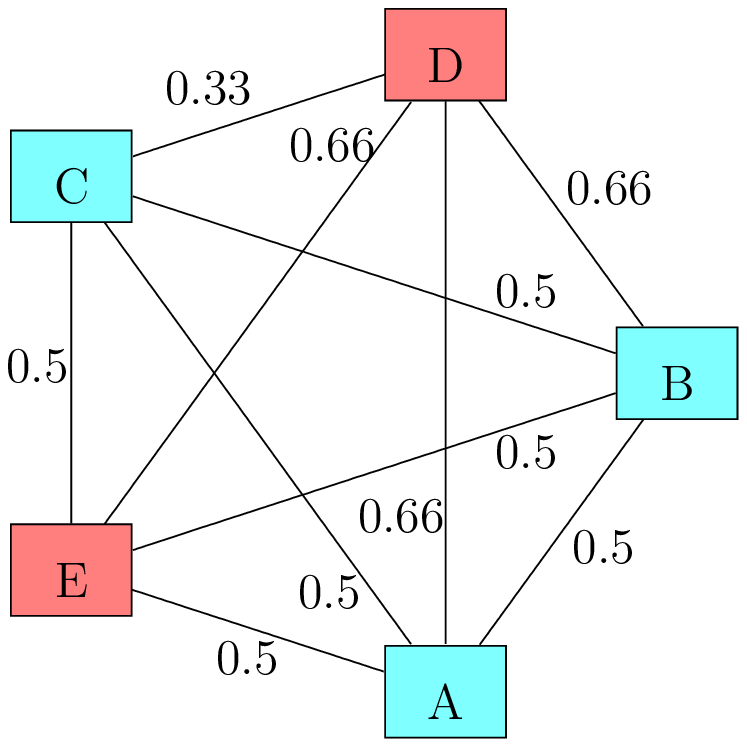}
  \end{center}
  \caption{[left] An example five-state filter.  (This filter was originally
  generated by the algorithm of O'Kane and Shell~\cite{OKaShe13}.  [right] The
  corresponding complete graph, along with an improper coloring \changed{of} that graph with
  two colors.}
  \label{fig:donut-plot5-1}
\end{figure}

The idea is then to assign a color $c(v)$ to each vertex $v$ of $C(F)$, using
at most $k$ colors, while minimizing the worst case over all vertices, of the
total weight of edges to same-colored neighbors.
That is, we want assign $k$ colors to the vertices of $C(F)$ in a way that
minimizes this objective function:
  \begin{equation}
    O(F, c) = \max_{v \in V(C(F))} \sum_{\{u \in V(C(F)) - \{ v \}  \: \mid \: c(v) = c(u)\}} w(u,v),
  \end{equation}
in which $V(C(F))$ denotes the vertex set of $C(F)$.  This optimization
problem, which is known as the \emph{Threshold Improper Coloring Problem}, has
been addressed by Araujo, Bermond, Giroire, Havet, Mazauric and
Modrzejewski~\cite{Araujo}.
Though the problem is NP-hard even to approximate---Note that an efficient
algorithm for this problem could be used to build an efficient algorithm for
the standard proper graph coloring problem---that prior research
presents a randomized heuristic algorithm that performs well in most cases.
We use that algorithm to color $C(F)$.

\subsection{Randomized voting for improper filter reduction}\label{sec:voting}
Next, we form the reduced filter $F'$ by `merging' the states in $F$
corresponding to each group of same-colored nodes in $C(F)$ into a single state in
$F'$.
The process is somewhat analogous to pairwise merging process
described in Section~\ref{sec:merge}, but must account for some important
differences.
Most importantly, because we want to merge the states within each of the $k$
color groups simultaneously, when selecting the edges in the reduced filter
$F'$, it only needs to consider which state in $F'$---that is, which color
group---should be the target of that edge, rather than making a finer-grained
selections of some state in some partially-reduced version of $F$.

Note, however, that the states in each color group may not agree on
which $F'$ state should be reached under each observation.  When, for a given observation $y$, such
disagreements occur, we resolve them in a randomized way.
The algorithm selects, using a uniform random distribution, one of the $F$
states in this color group that has \changed{an} edge labeled with $y$, and adds a
transition in $F'$ the state corresponding to the color group reached by that
state.  This forms a kind of `weighted voting,' in which it is more likely to
select transitions that correspond to larger number of states in the group.

Similarly, we select the output color of each state in $F'$ by randomly
selecting one of its constituent states and using its color as the output for
the combined state.
Continuing the example from Figure~\ref{fig:donut-plot5-1}, observe that to
create reduced filter with two states, one of those states will correspond to
the original filter's A, B, and C states, and other one will consist of D and
E.  Since D and E have different colors in the original filter, the algorithm
\changed{will} assign the new DE state to each of the two possible output colors with
equal probability.

Because the final filters produced by this process are not deterministic, we
repeat the reduction several times and return the best filter resulting from
those iterations.

%
\section{Implementation and Experimental Results}\label{sec:exp}

We have implemented Algorithms~\ref{alg:Dh}--\ref{alg:Global} in Python.
The experiments described below were executed on a GNU/Linux computer with a
3GHz processor.
Throughout, we used $\epsilon=0.035$ in Algorithms~\ref{alg:Dh} and
\ref{alg:De}, $\cins=\cdel=\csub=1$ in Algorithm~\ref{alg:De}, and $r=400$ in
Algorithm~\ref{alg:Global}.

\subsection{Distance and run time for varying filters}\label{sec:exp1}
First, we consider a family of filter reduction problems originally described
by Tovar, Cohen, and LaValle~\cite{tovar09beams}.  In these problems, a pair of
robots move through an annulus-shaped environment, occasionally crossing a beam
sensor that detects a crossing of that beam has occurred, but cannot
detect which robot has crossed, nor the direction of that crossing.  The filter
should process these observations and output 0 if the robots are together, or 1
if the robots are separated by at least one beam.  (This is a different and
more challenging problem than the single-robot variant described in
Section~\ref{sec:intro}.)

We varied the number of beam sensors from 3 to 9, and tested two equivalent
filters for each number of beams: one `unreduced' na\"\i{}ve filter, formed by
directly computing the sets of possible states after each observation, and a
smaller `reduced' filter produced by applying the algorithm of O'Kane and
Shell~\cite{OKaShe13} to the unreduced versions.  These reduced filters have the
same behavior as their unreduced counterparts, but are the smallest filters for
which that equivalence holds.

For each of these $7 \cdot 2 = 14$ filters, we executed both
Algorithm~\ref{alg:Greedy} and Algorithm~\ref{alg:Global}, using both $D_h$
and $D_e$ as the underlying metric for each algorithm.  The target filter size
was set to $k=2$, the maximal meaningful reduction, for each of these trials.
The results appear in Figure~\ref{fig:DHD} for Hamming distance and in
Figure~\ref{fig:DED} for edit distance.
\begin{figure}[t]
   \begin{center}
	  \includegraphics[width=\columnwidth]{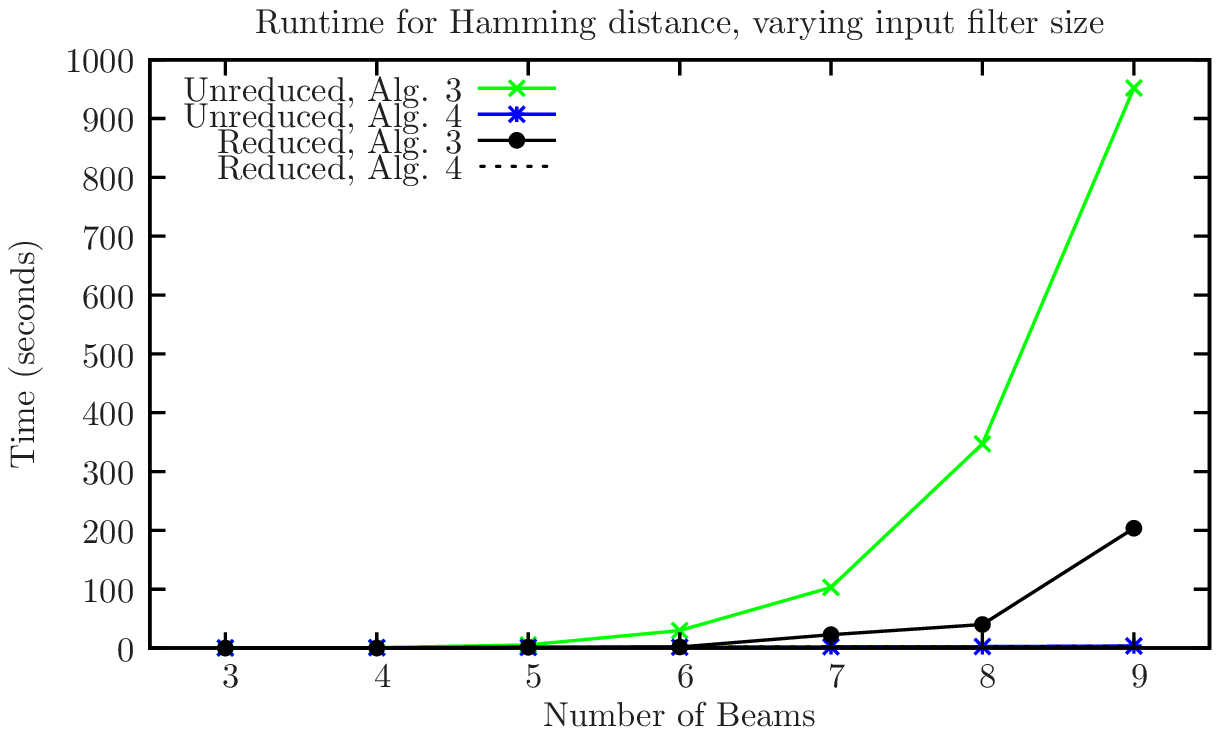}

    \bigskip
    \bigskip

   	\includegraphics[width=\columnwidth]{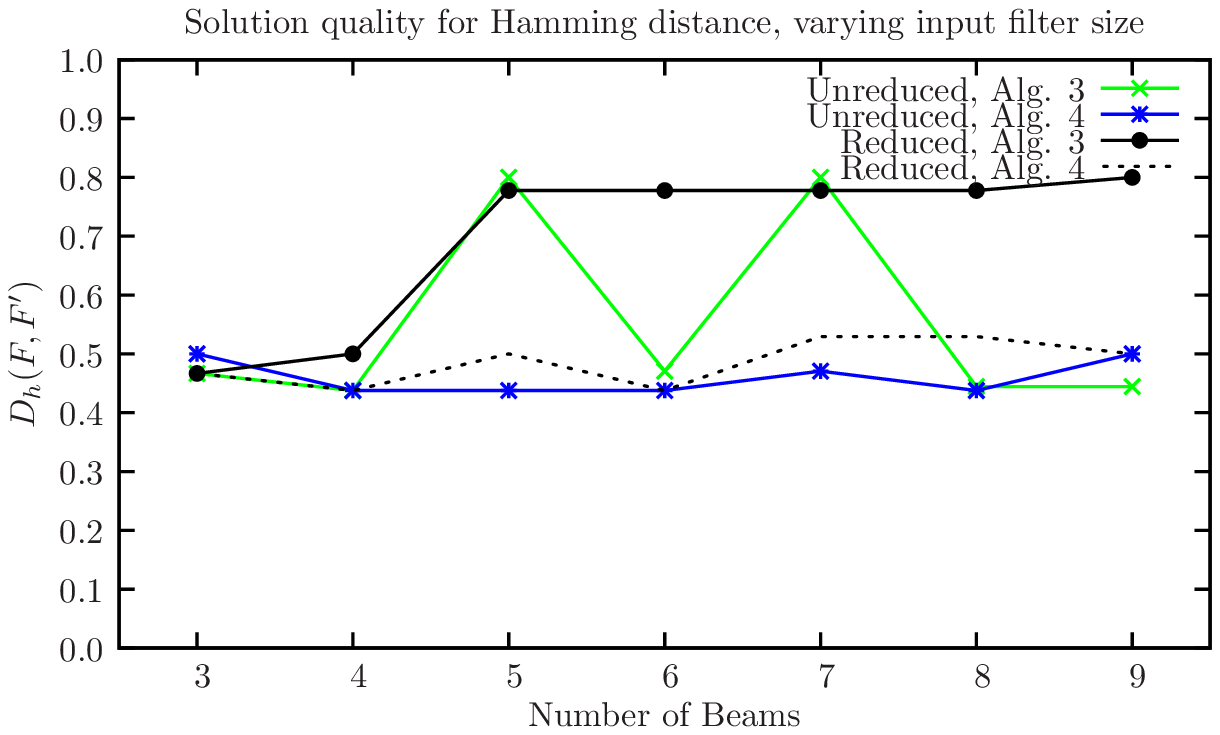}
   \end{center}
  \caption{Results for reduction of annulus filters under Hamming distance using
  Algorithm~\ref{alg:Greedy} and Algorithm~\ref{alg:Global}. [top] Run time.
  [bottom] Final distance $D_h(F, F')$ between original and reduced filters.}
	\label{fig:DHD}
\end{figure}

\begin{figure}[t]
   \begin{center}
	  \includegraphics[width=\columnwidth]{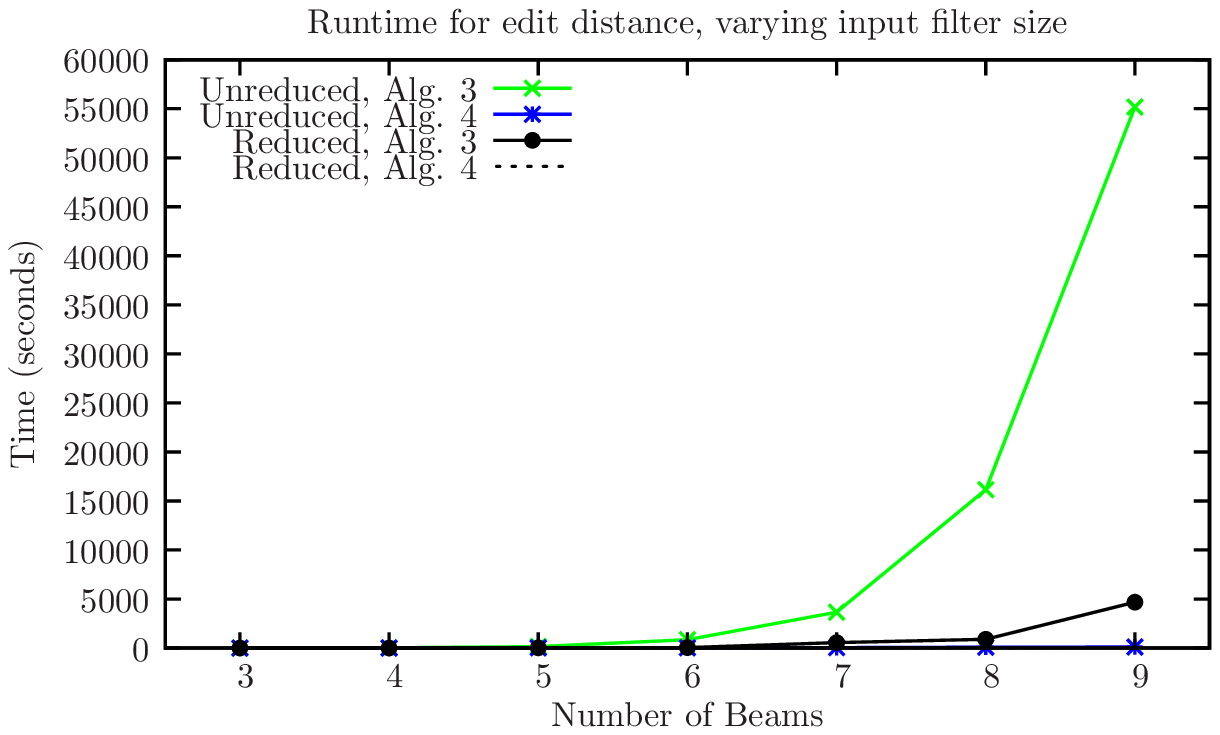}

    \bigskip
    \bigskip

   	\includegraphics[width=\columnwidth]{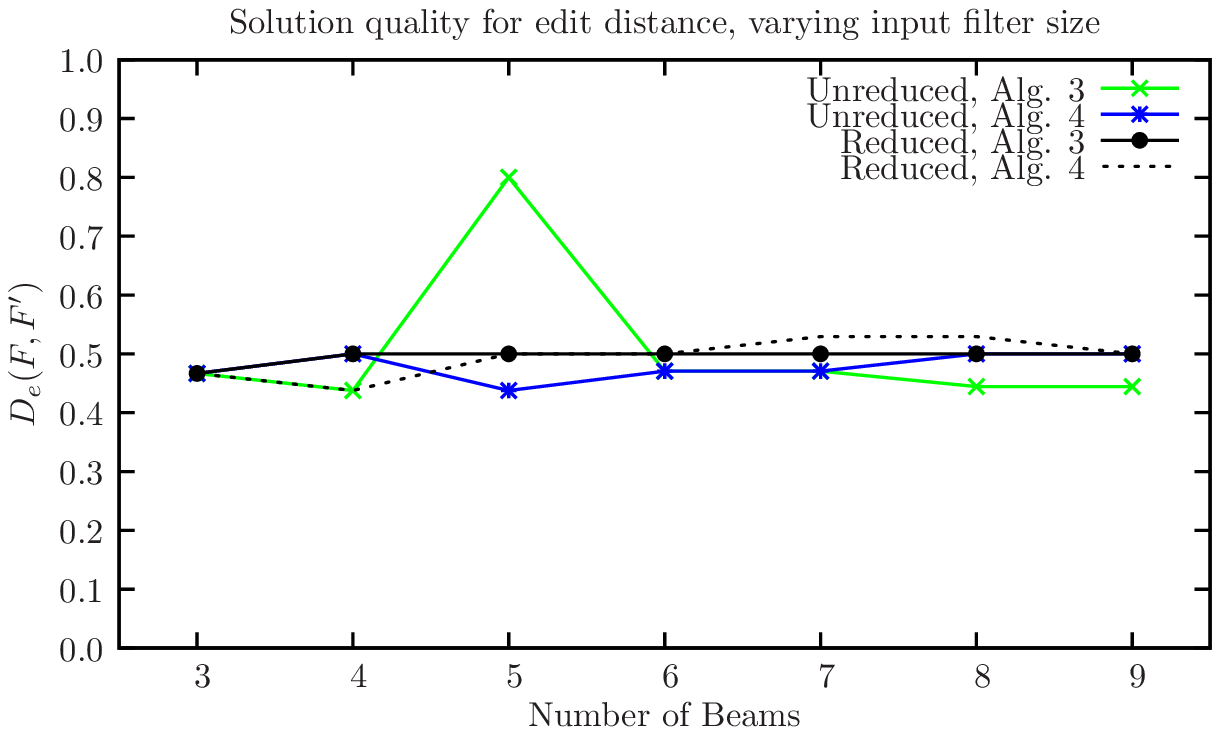}
   \end{center}
  \caption{Results for reduction of annulus filters under edit distance using
  Algorithm~\ref{alg:Greedy} and Algorithm~\ref{alg:Global}. [top] Run time.
  [bottom] Final distance $D_e(F, F')$ between original and reduced filters.}
	\label{fig:DED}
\end{figure}

These results show that the final solution quality is somewhat better for
Algorithm~\ref{alg:Global} in many cases.  The global approach of
Algorithm~\ref{alg:Global} is also faster the greedy sequential reduction of
Algorithm~\ref{alg:Greedy}.  \ambiguty{The difference, which is especially pronounced as the
filter size grows large}, is explained by the fact that
Algorithm~\ref{alg:Global} uses its subroutine for distance between
filters---that is, Algorithm~\ref{alg:Dh} or \ref{alg:De}---only once, rather
than many times for each reduction.

Note that the computation time is substantially shorter for Hamming
distance than for edit distance, resulting from the extra nested loop in
Algorithm~\ref{alg:De} that is not needed in Algorithm~\ref{alg:Dh}.

\clearpage 

\subsection{Varying the target size for a single filter}

Next, we evaluated the impact of the target size $k$ on the algorithms'
performance.  We used the two-agent eight-beam filter described above in
Section~\ref{sec:exp1}, in both its unreduced and reduced forms.  We varied $k$
from 2 to 10 and executed each of the algorithms for all $k$.
Figures~\ref{fig:D8HD} and \ref{fig:D8ED} show the results for Hamming distance
and edit distance, respectively.
The results show that, across all cases in this range, the run time is almost entirely
unaffected by the target size.  The solution quality shows a slight improving
trend as the target size increases, as one might expect.

\begin{figure}[t]
  \begin{center}
    \includegraphics[width=\columnwidth]{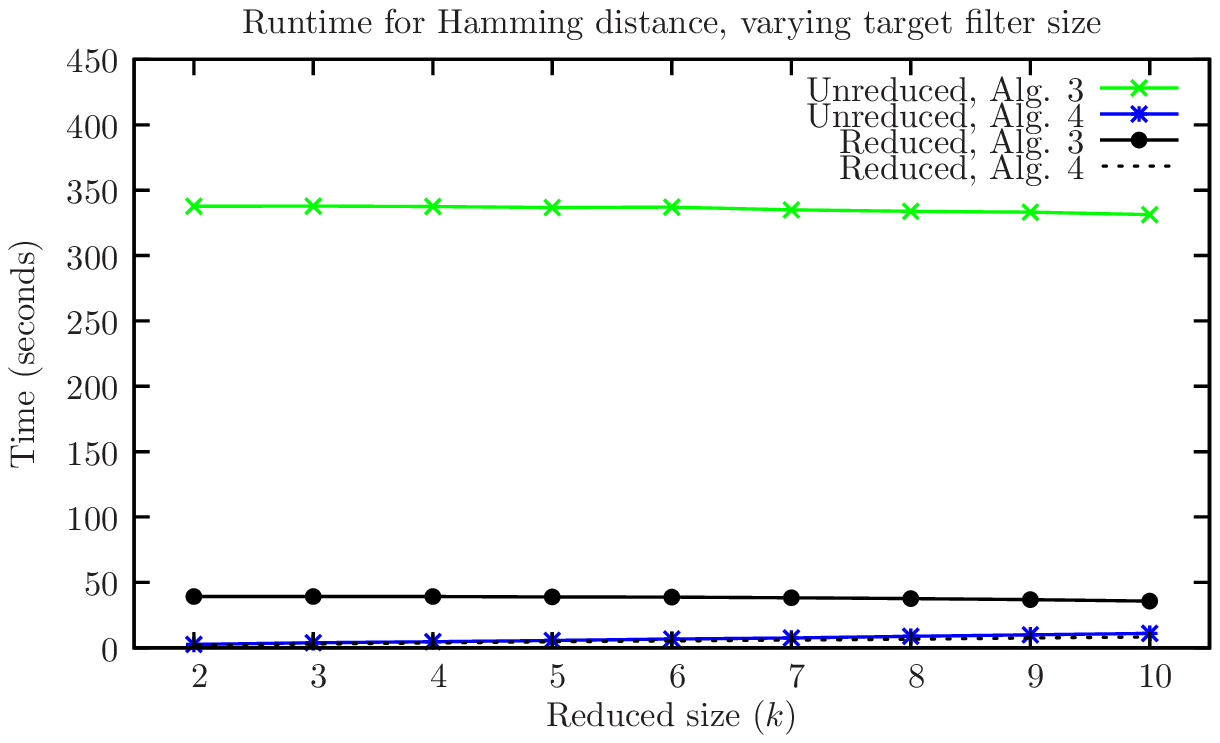}

    \bigskip
    \bigskip

    \includegraphics[width=\columnwidth]{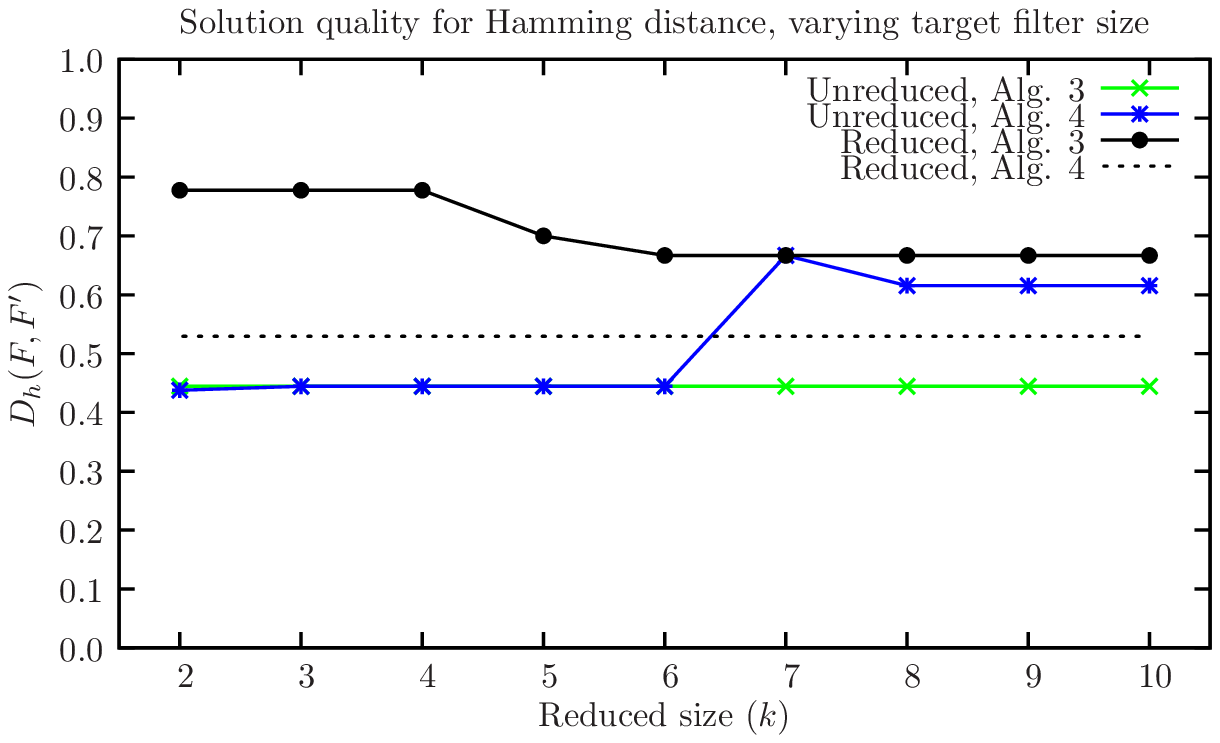}
  \end{center}
  \caption{Results for reduction of the eight-beam, two-robot annulus filter
  for varying target filter sizes under Hamming distance.  [top] Run time.
  [bottom] Final distance $D_h(F, F')$ between original and reduced filters.}
	\label{fig:D8HD}
\end{figure}

\begin{figure}[t]
  \begin{center}
    \includegraphics[width=\columnwidth]{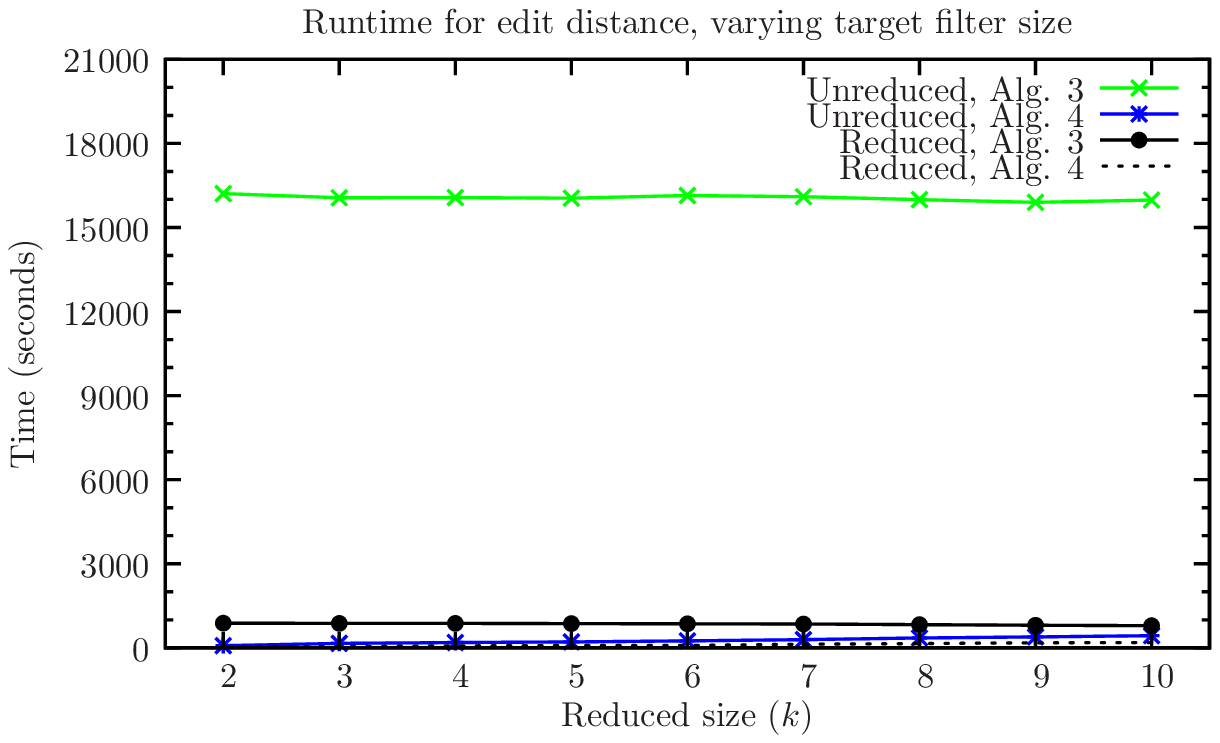}

    \bigskip
    \bigskip

    \includegraphics[width=\columnwidth]{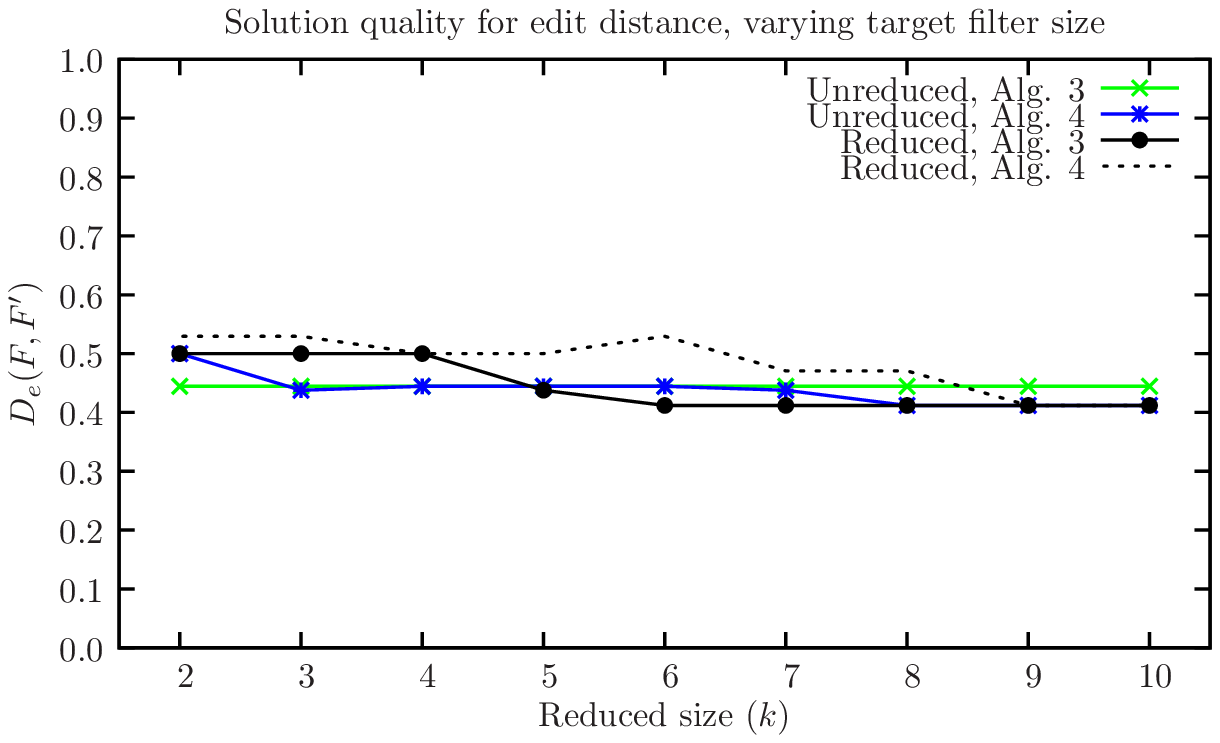}
  \end{center}
  \caption{Results for reduction of the eight-beam, two-robot annulus filter
  for varying target filter sizes under edit distance.  [top] Run time.
  [bottom] Final distance $D_e(F, F')$ between original and reduced filters.}
	\label{fig:D8ED}
\end{figure}

Finally, to confirm that these results generalize to other kinds of filters, we
repeated the analysis for the `L-shaped corridor' filtering problem introduced
by LaValle~\cite{Lav06}, which also appears in O'Kane and
Shell~\cite{OKaShe13}.  In this problem, a sensorless robot moves along a long
corridor with a right angle midway through it.  The robot moves in steps that
unpredictably vary between 1 or 2 steps.  The filter's goal is to output 0 when
the robot reaches the end of the corridor, or 1 otherwise.  This problem is
interesting because the size of the unreduced filters grows exponentially with
the corridor's length, whereas the reduced filter size grows only linearly.  In
this case, we use only the unreduced filters, because the reduced filter is too
small to be of interest.
Figures~\ref{fig:L3KHD} and \ref{fig:L3KED} show the results, which follow the
same general trends as the previous tests.  Of particular note that the
resulting filters are the same for Hamming distance and edit distance,
reflecting \changed{the} fact that insert or delete operations are not useful in this
problem, since the relevant information at each step is the furthest possible
state from the goal.  As a result, the error rates are identical between the
two metrics.  This is a clear instance in which Hamming distance, by virtue of
the faster run time of Algorithm~\ref{alg:Dh} is a better choice than edit
distance.

\begin{figure}[t]
  \begin{center}
    \includegraphics[width=\columnwidth]{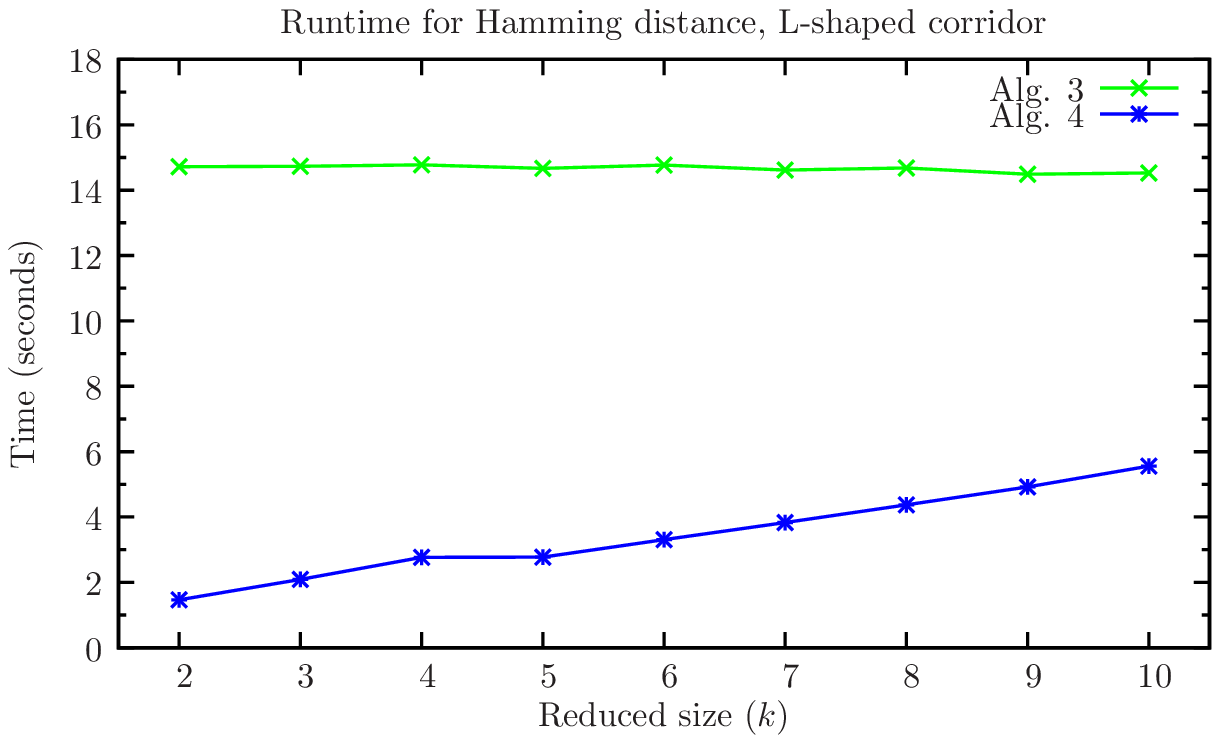}

    \bigskip
    \bigskip

    \includegraphics[width=\columnwidth]{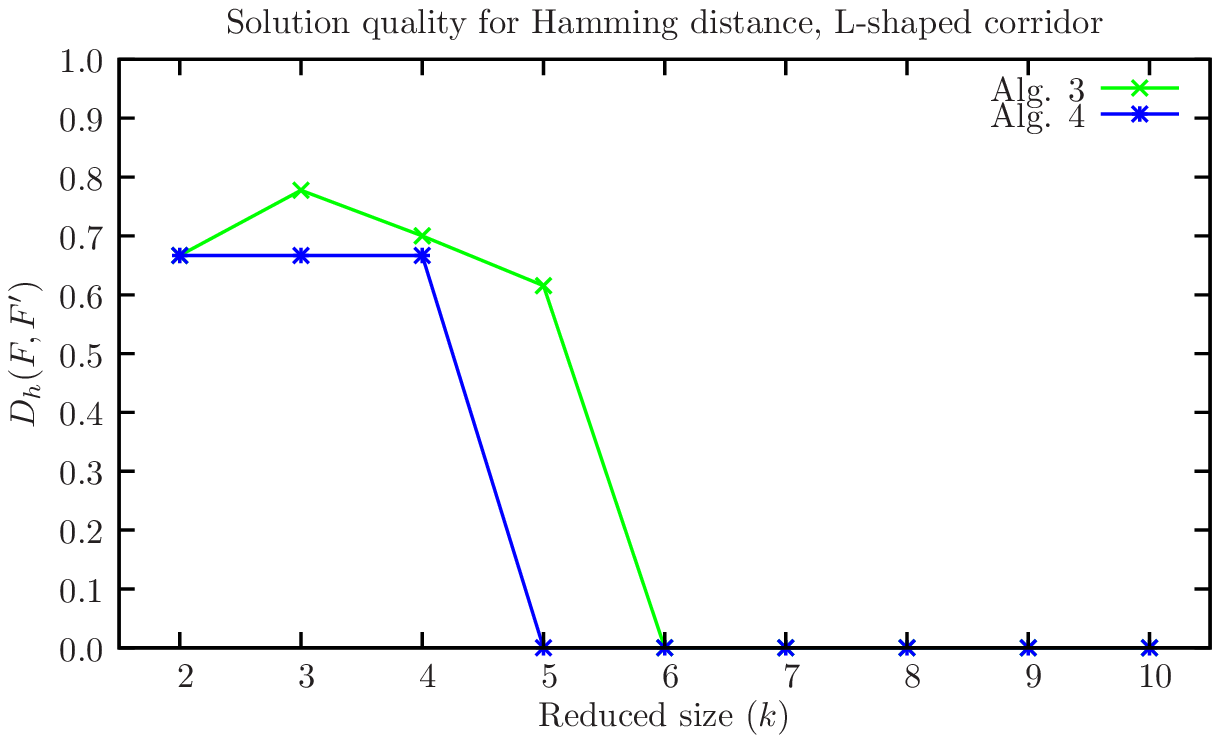}
  \end{center}
  \caption{Results for reduction of L-shaped corridor filter for varying target
  filter sizes under Hamming distance.  [top] Run time.  [bottom] Final
  distance $D_h(F, F')$ between original and reduced filters.}
	\label{fig:L3KHD}
\end{figure}

\begin{figure}[t]
  \begin{center}
    \includegraphics[width=\columnwidth]{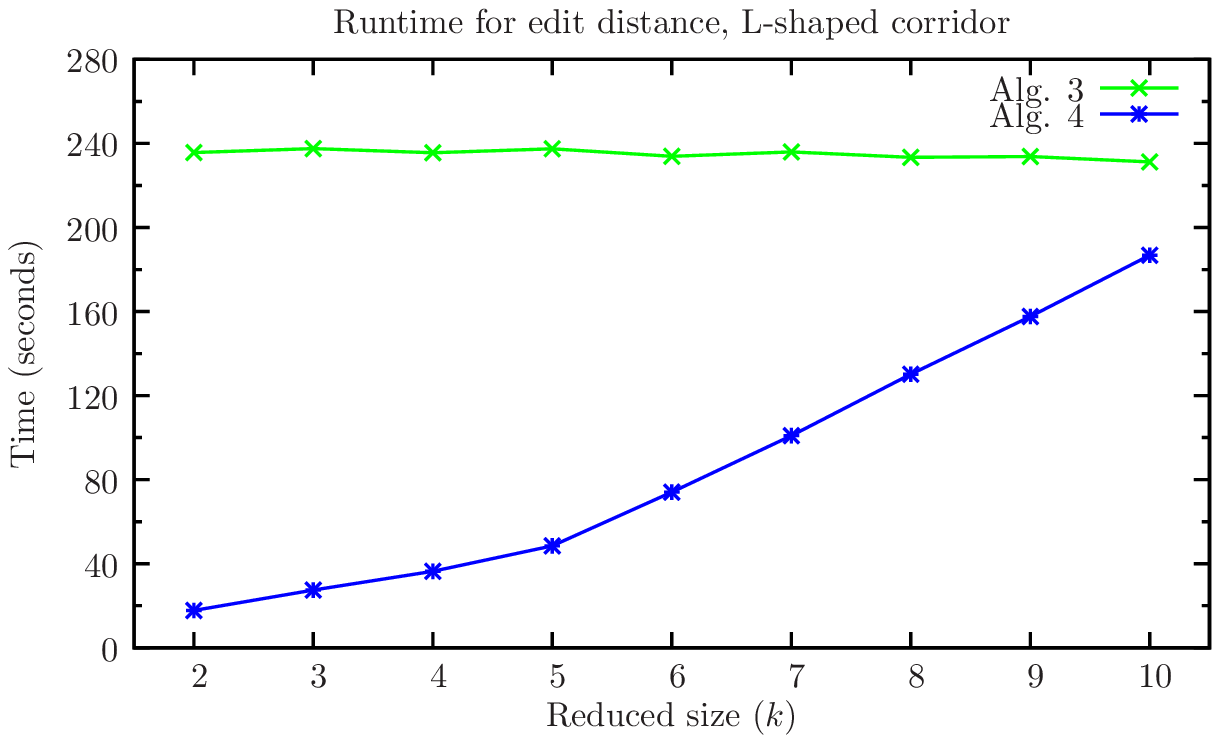}

    \bigskip
    \bigskip

    \includegraphics[width=\columnwidth]{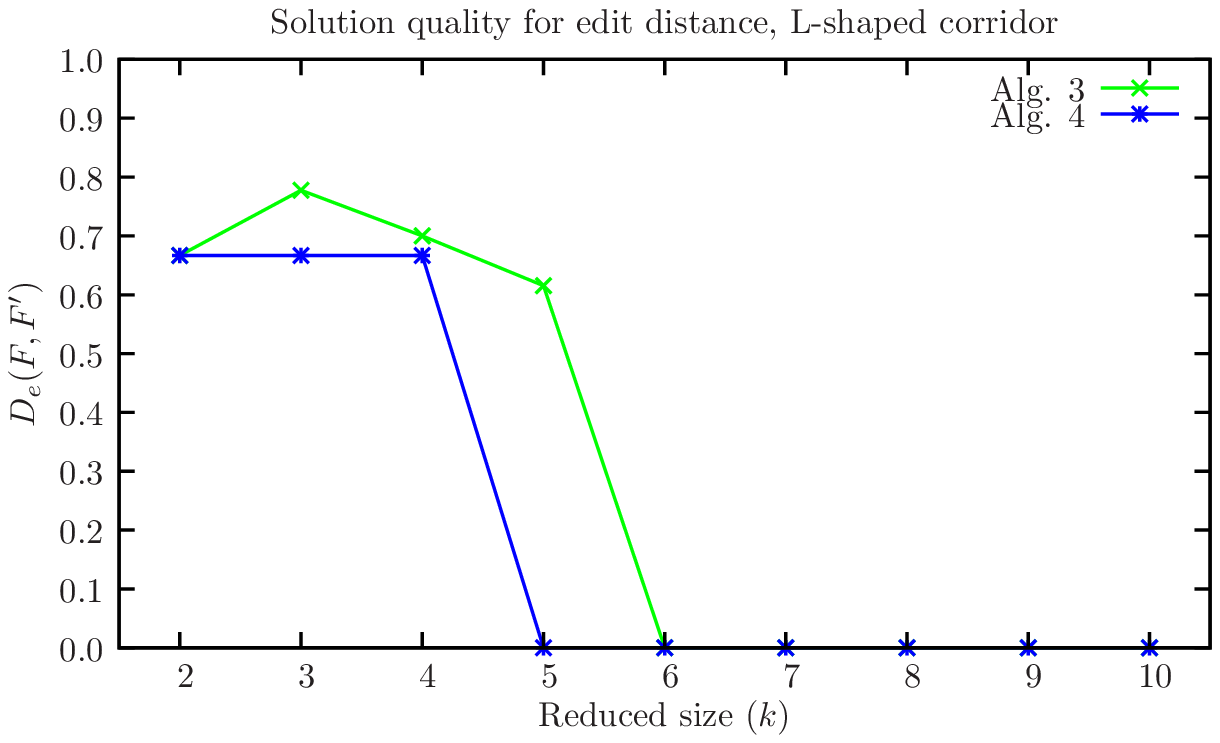}
  \end{center}
  \caption{Results for reduction of L-shaped corridor filter for varying target
  filter sizes under edit distance.  [top] Run time.  [bottom] Final
  distance $D_e(F, F')$ between original and reduced filters.}
	\label{fig:L3KED}
\end{figure}

\clearpage 

\section{Conclusion}\label{sec:conc}
In this paper, we introduced \changed{two metrics} to measure the similarity between two
filters and presented dynamic programming algorithms for computing \changed{these
measurements}. Then we presented two algorithms to reduce a filter to a given
size automatically.

There exist some future directions to extend this work.  Most directly, a good
extension of this work would be presenting more efficient algorithms.  In
particular, we anticipate that Algorithm~\ref{alg:Greedy} can be accelerated by
reducing the number of filter distance queries it makes, possibly by borrowing
the self-similarity idea from Algorithm~\ref{alg:Global}.
In addition, finding additional criteria for improving the merge operation, or
even replacing that operation with some other form of reduction is another possible
improvement.

Though we proved that \textsc{Improper-FM} is NP-hard, it is worthy of study to
determine whether it can be approximated efficiently.  It may also be the case
that certain interesting special classes of filters are efficiently solvable.

Another direction is to investigate whether \textsc{Improper-FM} is fixed
parameter tractable.  That is, does there exist a algorithm whose run time is
polynomial in the input size, but exponential in some other parameter that
measures the complexity of a given instance.

\section*{Acknowledgment}
The authors are grateful to Dylan Shell for helpful feedback on an earlier
version of this work.  This material is based upon work supported by the
National Science Foundation under Grant Nos. IIS-0953503 and IIS-1526862.

\end{document}